\documentclass[10pt,journal,compsoc]{IEEEtran}
\usepackage{amsmath,amssymb,amsfonts}
\usepackage[ruled,linesnumbered]{algorithm2e} 
\usepackage{array}
\usepackage{subcaption}
\usepackage{textcomp}
\usepackage{stfloats}
\usepackage{url}
\usepackage{verbatim}
\usepackage{graphicx}
\usepackage{enumitem}
\usepackage{cite}
\usepackage{svg}
\usepackage{hyperref}
\usepackage{graphicx}
\usepackage{epstopdf}
\usepackage{svg}
\usepackage{booktabs}
\usepackage{bm}
\usepackage{amsmath}
\usepackage{amsthm}
\usepackage{caption}
\usepackage{graphicx}   
\usepackage{subcaption} 
\usepackage{xcolor}
\colorlet{red}{black}
\colorlet{c51red}{black}

\newenvironment{revblock}
  {\begingroup}
  {\endgroup}

\newtheorem{lemma}{Lemma}
\newtheorem{theorem}{Theorem}
\newtheorem{definition}{Definition}
\newtheorem{corollary}{Corollary}
\theoremstyle{remark}
\newtheorem{remark}{Remark}

\graphicspath{{Fig/}}
\begin{document}

\title{FlowCritic: Bridging Value Estimation with Flow Matching in Reinforcement Learning}
\author{Shan Zhong, Shutong Ding, He Diao, Xiangyu Wang,  Kah Chan Teh, and Bei Peng*
	\thanks{This work was supported by NSFC, China (51975107), and in part by the China Scholarship Council under Grant CSC202406070141. Corresponding author: Bei Peng.}
	\thanks{S. Zhong,  He Diao, Xiangyu Wang, and B. Peng are with the School of Artificial Intelligence, University of Electronic Science and Technology of China, Chengdu 611731, China (e-mail: 202211040927@std.uestc.edu.cn; 202311040616@std.uestc.edu.cn;202322040551@std.uestc.edu.cn; beipeng@uestc.edu.cn;). Bei Peng is also with Sichuan EIR Technology Co. Ltd.}
	\thanks{S. Ding is with  the School of Information Science and Technology, ShanghaiTech University. (dingsht@shanghaitech.edu.cn)}
	\thanks{K. C. Teh is with the School of Electrical and Electronic Engineering, Nanyang Technological University, Singapore.(e-mail: ekcteh@ntu.edu.sg).}
}

\maketitle

\begin{abstract}
Reliable value estimation serves as the cornerstone of reinforcement learning (RL) by evaluating long-term returns and guiding policy improvement, significantly influencing the convergence speed and final performance. Existing works improve the reliability of value function estimation via multi-critic ensembles and distributional RL, yet the former merely combines multi point estimation without capturing distributional information, whereas the latter relies on discretization or quantile regression, limiting the expressiveness of complex value distributions.
Inspired by flow matching's success in generative modeling, we propose a generative paradigm for value estimation, named FlowCritic. Departing from conventional regression for deterministic value prediction, FlowCritic leverages flow matching to model value distributions and generate samples for value estimation.
Specifically, FlowCritic formulates the target value distribution through distributional Bellman operators, and trains a velocity field network to learn the continuous probability flow from a tractable prior distribution to the target distribution, thereby enabling flexible modeling of arbitrarily complex value distributions.
More importantly, FlowCritic introduces the coefficient of variation (CoV) of the generated value distribution to quantify the noise level of training samples. Based on CoV, FlowCritic assigns adaptive weights to prioritize low noise samples, thereby reducing the variance of policy gradient during backpropagation. Additionally, FlowCritic incorporates truncated sampling and velocity field clipping mechanisms to ensure training stability.
Our theoretical analysis demonstrates FlowCritic's convergence and advantages. Extensive experiments on 12 IsaacGym benchmarks demonstrate FlowCritic's superiority over existing RL baselines, while its successful deployment on a real quadrupedal robot platform validates its effectiveness in practical physical systems.
Notably, FlowCritic is the first approach to integrate flow matching into value distribution modeling for RL, and offers a new perspective on the effective exploitation of distributional information.

\textbf{Code:} \url{https://github.com/smartXiaoz/kalman_flowc.git}\footnote{The repository is currently private and will be made public upon publication.}.
	
\end{abstract}


\section{Introduction}

Reinforcement Learning (RL) has achieved remarkable success in recent years and has been widely applied to numerous challenging domains, such as robotic control~\cite{ethro,ethro2,tro,pami_robo,pami_rl}, autonomous driving~\cite{zhao2024survey, pami_video},  intelligent wireless networks with integrated sensing and communication~\cite{Illi1,Magbool}, and post-training of large language models~\cite{llm,guo2025deepseek}. The value function lies at the core of RL, as it directly determines the algorithm's ability to efficiently evaluate the long-term returns of states or actions, thereby dictating its convergence speed and final performance~\cite{sutton1,zhong2025mad3pg, arulkumaran2017deep, prudencio2023survey}. However, value estimation faces the dual challenges of bias and high variance, stemming from the inherent stochasticity of the environment, exploration noise in the policy, and the intrinsic errors of function approximation~\cite{mnih2015human, sutton2, wu2020reducing, wu2018variance,highbias}. 

{\color{red}To improve the robustness of value estimation, prior studies have explored several directions, including bias-variance trade-off control~\cite{highbias}, overestimation mitigation~\cite{hasselt2010double,td3}, multi-critic architectures~\cite{averagedqn,maxmin,sacn,robotkeyframing}, and distributional reinforcement learning (DRL)~\cite{bellemare,drlbook,QR-DQN,IQN,dsac,dsac2,TQC}. In particular, DRL models the full value distribution rather than only its expectation, thereby enhancing the expressive capability of value representation.}

{\color{red}Nevertheless, improving RL performance through better value functions still faces two key challenges:}
\begin{enumerate}
	 \item{Discrete Approximation of Value Function:} Both multi-critic ensembles and DRL fundamentally rely on regression-based approximation to learn deterministic value mappings~\cite{sutton1, farebrother2024stop}. Each network in an ensemble produces a fixed point estimate, with the ensemble merely aggregating these finite outputs. While DRL extends value functions to distributional representations, the methods that rely on Gaussian assumptions or quantile discretization inherently constrain the distribution family, thereby limiting the flexibility to capture complex value distributions~\cite{drlbook}.
	\item{Unexploited Value Distribution Statistics:} Existing methods typically extract only the mean or a conservative estimate from value distributions for policy guidance, which neglects higher-order statistical features of the value distribution~\cite{TQC, sun2024does}. These statistical features may characterize the stochasticity of interactions with the environment and the noise level during exploration for each learning sample, yet they are not fully exploited in RL, which limits the potential for policy optimization.
\end{enumerate}

Inspired by the recent success of flow matching in generative modeling~\cite{DDPM,ddim, flow, liu2022flow}, we propose a general framework for value distribution modeling, named FlowCritic, which represents the value function as a samplable generative model. Compared to conventional regression-based methods, FlowCritic leverages the expressive power of flow matching to approximate arbitrarily complex value distribution, and exploits generative sampling to effectively extract distributional statistics to guide RL training.
Specifically, FlowCritic formulates target value distribution using distributional Bellman operators. Subsequently, a velocity field network is trained to match the instantaneous velocity that transports the prior Gaussian distribution to the target distribution, which is then integrated via an ordinary differential equation (ODE) to generate the value estimates. Furthermore, truncated sampling and velocity field clipping are introduced to enhance training robustness. Truncated sampling discards extreme high values to reduce outlier effects, while velocity field clipping limits the magnitude of velocity field updates to stabilize the flow matching model’s convergence.

Moreover, inspired by signal processing principles, we re-examine the exploitation of distributional information. 
A key insight is that return estimates across states have varying reliability. For instance, the value network produces more reliable estimates for trained samples than for unseen ones, and environment stochasticity further affect estimate quality, where we use noise levels to characterize the quality differences across different state estimates. 
However, existing algorithms treat all training samples equally and ignore the noise levels of value estimates, which causes high-noise samples to propagate excessive variance during gradient backpropagation, degrading training stability and hindering convergence. Accordingly, FlowCritic introduces the coefficient of variation (CoV), a second-order statistic of the value distribution, to quantify noise levels across training samples, where CoV is derived from multiple value estimates generated via ODE integration. 
\begin{revblock}
Notably, the CoV relies on the sampled value distribution provided by FlowCritic, whereas conventional point value critics only output a single scalar estimate and therefore cannot directly quantify noise levels across training samples.
\end{revblock}
With this measure, low-CoV samples provide more reliable value estimates than high-CoV samples. Accordingly, FlowCritic assigns higher weights to low-CoV samples and lower weights to high-CoV ones through adaptive weighting, enabling reliable value estimates to guide policy optimization, thereby reducing variance of policy gradients and enhancing sample efficiency in RL training.
Our contributions are summarized as follows:
\begin{enumerate}
	\item We first introduce flow matching into RL value modeling, proposing a generative paradigm for modeling arbitrarily complex value distributions. Furthermore,  we provide theoretical convergence analysis for this approach.
	\item FlowCritic innovatively exploits the CoV from value distributions to weight training samples adaptively, thereby reducing gradient variance. We also provide theoretical analysis of this weighting advantage. Additionally, truncated sampling and velocity field clipping mechanisms ensure training robustness.
	\item Extensive experiments on twelve IsaacGym benchmarks demonstrate that our instantiation of FlowCritic significantly outperforms baselines. Its successful deployment on a real quadrupedal robot further validates practical effectiveness in physical systems.
\end{enumerate}
{\color{red}The remainder of this paper is organized as follows. After reviewing related work in Section~\ref{rel} and introducing the necessary preliminaries in Section~\ref{pre}, Section~\ref{met} presents the implementation of the proposed FlowCritic paradigm. Sections~\ref{the} and~\ref{exp} then provide the theoretical analysis and experimental results, respectively. Finally, Section~\ref{con} concludes the paper and discusses future work.}
\begin{revblock}
\section{Related Work}
\label{rel}
\subsection{Generative Models for Policy Learning}
In recent years, generative models such as diffusion and flow matching models have emerged as a promising direction for policy learning~\cite{reuss2025flower,brahmanage2023flowpg, QVPO}. In offline reinforcement learning and robot policy learning, these methods directly model the action distribution, which provides a natural way to capture multimodal action structures in continuous control and overcomes the limited expressiveness of conventional policy parameterizations for complex action distributions~\cite{DiffusionQL,Idql,DiffusionPolicy}. More recently, generative policies have also been explored in online policy optimization, where diffusion or flow models are learned through ongoing interaction with the environment rather than from fixed offline datasets~\cite{ding2026genpo,mcallister2026flow}.  These methods are primarily designed to model action distributions, whereas FlowCritic uses flow matching to model the return distribution conditioned on the state and leverages the resulting distributional statistics to improve value estimation and policy updates.

\subsection{Robust Value Estimation and Multi-Critic Methods}
Prior work has sought more reliable value estimates by stabilizing scalar critics. In policy gradient methods, generalized advantage estimation provides practical control over the bias-variance trade-off through exponentially weighted temporal-difference residuals~\cite{highbias}. In value-based RL, Double Q-learning reduces overestimation by decoupling action selection from target evaluation~\cite{hasselt2010double}, while TD3 improves stability in continuous control through clipped double critics and delayed policy updates~\cite{td3}. A complementary approach uses multiple critics. Averaged-DQN reduces target variance by averaging a sequence of historical Q-networks~\cite{averagedqn}. Maxmin Q-learning and SAC-N maintain several critics and use conservative aggregation to suppress overestimation or capture uncertainty~\cite{maxmin,sacn}. Similar ideas have also been used in robot learning, where combining several value estimates can improve training stability~\cite{robotkeyframing}. Despite of their effectiveness, these methods remain within the scalar value estimation paradigm, since each critic predicts a point estimate and the final target is still constructed from deterministic outputs.

\subsection{Distributional Reinforcement Learning}
DRL moves beyond point estimation by modeling the full return distribution rather than only its expectation~\cite{bellemare,drlbook}. In discrete action settings, C51 represents returns with fixed support atoms and projects Bellman updates back onto this support~\cite{bellemare}. QR-DQN replaces fixed supports with a finite set of learned quantiles obtained through quantile regression, yielding a more flexible representation~\cite{QR-DQN}. IQN extends this idea by learning the quantile function implicitly, which allows the model to represent returns over a continuous range of quantile fractions~\cite{IQN}. In continuous control, DSAC introduces a Gaussian distributional critic into SAC to represent value uncertainty and reduce estimation errors~\cite{dsac,dsac2}. TQC uses multiple quantile critics together with a conservative target construction to curb overestimation under function approximation~\cite{TQC}. Recent work has also examined how return distributions can benefit optimization~\cite{sun2024does}.

{\color{c51red}Despite these advances, most DRL methods still rely on predefined distributional forms, such as fixed atoms, Gaussian families, or finite quantile samples. Recent generative value distribution methods provide a more flexible alternative. Bellman Diffusion preserves Bellman linearity in the distribution space~\cite{li2024bellman}, VDRL learns the value distribution with a diffusion model~\cite{hu2026value}, DBC builds a diffusion bridge distributional critic~\cite{ding2026distributional}, and SDPG models the return distribution with reparameterized samples~\cite{singh2022sample}. In contrast, FlowCritic uses flow matching to model the state value return distribution with a flexible generative formulation and uses statistics of the learned distribution to reweight policy updates.}
\end{revblock}
\section{Preliminaries}
\label{pre}
\subsection{Online Reinforcement Learning}\label{preA}

RL problems are typically formulated as a Markov Decision Process (MDP)~\cite{sutton1998reinforcement}, represented by a 5-tuple \((\mathcal{S}, \mathcal{A}, P, r, \gamma)\), where \(\mathcal{S}\) and \(\mathcal{A}\) denote the state and action spaces, respectively, \(P(s_{t+1} | s_t, a_t)\) is the transition probability, \(r_t = r(s_t, a_t)\) is the immediate reward, and \(\gamma \in [0,1)\) is the discount factor. An agent interacts with the environment through a parameterized policy \(\pi_\theta(a_t|s_t)\), generating trajectories \(\tau = (s_0, a_0, r_0, s_1, \dots)\), with the objective of maximizing the expected cumulative discounted return:
$
J(\pi_\theta) = \mathbb{E}_{\tau \sim \pi_\theta} \left[ \sum_{t=0}^\infty \gamma^t r_t \right].
$
To achieve this objective, RL algorithms rely on value functions to evaluate and improve policies. Value functions can be categorized into two main types: the state value function $V^{\pi}(s) = \mathbb{E}_{\pi}\left[\sum_{t=0}^{\infty} \gamma^t r_t \mid s_0 = s\right]$, which estimates the expected return from a given state, and the action value function $Q^{\pi}(s,a) = \mathbb{E}_{\pi}\left[\sum_{t=0}^{\infty} \gamma^t r_t \mid s_0 = s, a_0 = a\right]$, which estimates the expected return after taking a specific action. These value functions form the foundation for different algorithmic paradigms.

Value-based methods primarily rely on action value functions for decision-making. For instance, methods like DQN learn $Q_{\phi}(s,a)$ to directly select optimal actions by minimizing TD errors~\cite{hasselt2010double, averagedqn}:
\begin{equation}
	\mathcal{L}_Q = \mathbb{E} \left[ \left(Q_{\phi}(s_t, a_t) - \left(r_t + \gamma \max_{a'} Q_{\phi'}(s_{t+1}, a')\right)\right)^2 \right].
\end{equation}
In contrast, policy gradient methods utilize state value functions to evaluate policy performance through advantage estimation. The advantage function $A^{\pi}(s,a) = Q^{\pi}(s,a) - V^{\pi}(s)$ measures how much better an action is compared to the average, commonly computed via GAE~\cite{highbias}:
\begin{equation}
	\hat{A}_t = \sum_{l=0}^{T-t-1} (\gamma \lambda)^l \delta_{t+l}, \quad \delta_t = r_t + \gamma V_{\phi}(s_{t+1}) - V_{\phi}(s_t),
\end{equation}
where $\lambda \in [0,1]$ controls the bias-variance trade-off. The empirical return target is constructed as:
\begin{equation}
	\hat{R}_t = \sum_{l=0}^{T-t-1} (\gamma \lambda)^l \delta_{t+l} + V_{\phi}(s_t).
	\label{ppor}
\end{equation}
The state value function $V_{\phi}(s)$ is then trained to fit this target by minimizing the mean squared error:
\begin{equation}
	\mathcal{L}_V = \mathbb{E}_{t} \left[ \left(V_{\phi}(s_t) - \hat{R}_t\right)^2 \right].
\end{equation}
It is evident that both Q-function and V-function are approximated through regression networks producing deterministic point estimates, which deprives them of the potential to exploit value distributional statistics for policy learning.

\subsection{Flow Matching Methods}
Flow matching is an emerging generative modeling method. It learns a time-dependent velocity field to establish continuous transport between a prior distribution and the target distribution~\cite{liu2022flow, albergo2022building, tong2024improving, flow}. Let the data distribution be $p_{\text{data}}(x)$ and the prior distribution be $p_{\text{prior}}(\varepsilon)$. A time-dependent interpolation path is defined as:
\begin{equation}
	o^{t} = t x + (1-t)\varepsilon, 
	\quad x \sim p_{\text{data}}, \ \varepsilon \sim p_{\text{prior}},
	\label{eq:flow_path}
\end{equation}
where the conditional velocity at time $t$ is given by
\begin{equation}
	u^{t}(o^{t}\mid x) = \frac{do^{t}}{dt} = x - \varepsilon .
\end{equation}
Since the same $o^{t}$ may be generated from different pairs $(x,\varepsilon)$, the conditional velocity $u^{t}(o^{t} \mid x)$ is not unique. To capture the averaged dynamics over all possible paths, flow matching introduces the concept of marginal velocity, which at time $t$ is defined as:
\begin{equation}
	u(o^{t},t) =
	\mathbb{E}_{p(u^{t}\mid o^{t})}\!\left[ u^{t}(o^{t}\mid x) \right],
	\label{eq:marginal_velocity}
\end{equation}
where $p(u^{t} \mid o^{t})$ denotes the conditional distribution of velocities given a fixed $o^{t}$.
During training, the objective is to make the neural network $f_\theta(o^{t},t)$ approximate the true marginal velocity $u(o^{t},t)$. The corresponding loss function is formulated as:
\begin{equation}
	\mathcal{L}_{\text{FM}}(\theta)
	= \mathbb{E}_{t,\, p_t(o^t)}
	\left[\|f_\theta(o^t,t)-u(o^t,t)\|^2\right].
	\label{eq:flow_matching_loss1}
\end{equation}
\begin{revblock}
	In practice, Eq.~\eqref{eq:flow_matching_loss1} cannot be optimized as it is. Both the marginal path density $p_t(o^t)$ and the marginal velocity $u(o^t,t)$ require integrations over the data distribution that are intractable. The density $p_t(o^t)$ is defined by marginalizing over $p_{\text{data}}(x)$. However, the analytical form of $p_{\text{data}}(x)$ is unavailable, and we only have samples from it. The velocity $u(o^t,t)$ depends on the posterior $p_t(x \mid o^t)$. This posterior also requires the normalizing constant $p_t(o^t)$. Conditional flow matching (CFM) avoids these difficulties by working with conditional quantities instead of marginal ones~\cite{flow}. It samples $x \sim p_{\text{data}}$ and $\varepsilon \sim p_{\text{prior}}$, constructs $o^t$, and regresses the model against the conditional velocity $u^{t}(o^{t}\mid x)$. Since $u(o^t,t)=\mathbb{E}[u^{t}(o^{t}\mid x)\mid o^t]$ and $f_\theta(o^t,t)$ depends only on $(o^t,t)$, the CFM objective differs from the FM objective only by a term independent of $\theta$. As a result, the two objectives have the same gradient with respect to $\theta$~\cite{flow}.
	\end{revblock}	
The training objective of CFM is therefore expressed as:
\begin{equation}
	\mathcal{L}_{\text{CFM}}(\theta) 
	= \mathbb{E}_{t, \, x \sim p_{\text{data}}, \, \varepsilon \sim p_{\text{prior}}}
	\left[ \, \| f_\theta(o^{t},t) - u^{t}(o^{t}\mid x) \|^2 \, \right].
	\label{eq:cfm_loss}
\end{equation}

Once trained, the flow matching model can leverage an ODE solver to efficiently generate data by integrating along the learned velocity field, starting from the prior distribution. In this work, we adopt $\varepsilon \sim \mathcal{N}(0, 1)$ as the prior.

\textbf{Notation:} To avoid ambiguity, we distinguish between two types of timesteps: the subscript \(t\) denotes the environment timestep in RL, e.g., \(s_t\) indicates the state at timestep \(t\); the superscript \(t\) denotes the continuous interpolation time in flow matching, e.g., \(o^{t}\) represents the variable along the flow at interpolation time \(t \in [0,1]\).

\section{Methods}
\label{met}
This section presents the implementation of the FlowCritic paradigm. Specifically, we instantiate it within the policy gradient framework using PPO as the base algorithm~\cite{ppo}, given its proven stability, scalability, and extensive adoption in robotics and complex control tasks. It is important to emphasize that FlowCritic is a general paradigm and can be readily integrated with various RL methods.
\subsection{Value Distribution Modeling via Flow Matching}
We extend the conventional expectation form
$
V^\pi(s) = \mathbb{E}_{\pi}\!\left[ \sum_{t=0}^\infty \gamma^t r_t \,\middle|\, s_0=s \right]
$
to a distributional perspective by defining the random return
\begin{equation}
	Z^\pi(s) \triangleq \sum_{t=0}^\infty \gamma^t r_t, \quad Z^\pi(s) \sim p^\pi(\cdot \mid s),
\end{equation}
where \(Z^\pi(s)\) is the random variable of discounted cumulative returns. This naturally induces the distributional Bellman expectation operator $\mathcal{T}^\pi$, which maps any conditional distribution $p \in \mathcal{M}$ to
\begin{equation}
	(\mathcal{T}^\pi p)(\cdot \mid s) \;\triangleq\; r(s,a) + \gamma Z(s'),
\label{bellman}
\end{equation}
with $Z(s') \sim p(\cdot \mid s')$. 

\begin{revblock}Hence, we construct a parameterized distribution $p^\pi_\theta(\cdot \mid s)$ based on flow matching to approximate the true return distribution $p^\pi(\cdot \mid s)$. Specifically, we define $p_\theta^\pi$ through a push-forward formulation
\begin{equation}
	p_\theta^\pi(\cdot|s) := \mathcal{F}_\theta(s, \cdot)_\# p_{\text{prior}},
	\label{eq:push_forward}
\end{equation}
where $\mathcal{F}_\theta(s, \varepsilon)$ maps a prior sample $\varepsilon \sim p_{\text{prior}}$ to a return sample by integrating the velocity field $f_\theta$ along the following ODE from $t=0$ to $t=1$:
\begin{equation}
	\frac{d o^{t}}{dt} = f_\theta(o^{t}, t, s), \quad o^0 = \varepsilon, \quad \mathcal{F}_\theta(s, \varepsilon) := o^1.
	\label{eq:flow_ode}
\end{equation}
By construction, $Z = \mathcal{F}_\theta(s, \varepsilon)$ with $\varepsilon \sim p_{\text{prior}}$ yields $Z \sim p_\theta^\pi(\cdot|s)$. While $p_{\text{prior}}$ can be arbitrary, we use the standard Gaussian $\varepsilon \sim \mathcal{N}(0,1)$ in this work, and approximate the ODE solution by the Euler method with 5 iterations and a fixed step size $\Delta t$:
\begin{equation}
	o^{t+\Delta t} = o^{t} + \Delta t \cdot f_\theta(o^t, t, s).
	\label{g2}
\end{equation}
Training $f_\theta$ requires target samples from $p^\pi$, which are unavailable in RL. Directly using the scalar target $\hat{R}_t$ would constrain the flow matching model to approximate only the expected return $\mathbb{E}[Z^\pi(s)]$ rather than the full distribution $p^\pi(\cdot|s)$. FlowCritic therefore introduces a distributional TD sample as the bootstrap target:
\begin{equation}
	\zeta_t = r_t + \gamma \mathcal{F}_\theta(s_{t+1}, \varepsilon_{t+1}) - \mathcal{F}_\theta(s_t, \varepsilon_t).
\end{equation}
From Eq.~\eqref{eq:push_forward}, $\mathcal{F}_\theta(s_{t+1}, \varepsilon_{t+1})$ is distributed as $p_\theta^\pi(\cdot|s_{t+1})$ when conditioned on $s_{t+1}$. Consider the core term $y_t^\theta := r_t + \gamma \mathcal{F}_\theta(s_{t+1}, \varepsilon_{t+1})$ in $\zeta_t$, which has the same distribution as $r(s,a) + \gamma Z(s')$ with $Z(s') \sim p_\theta^\pi(\cdot|s')$ in Eq.~\eqref{bellman}, yielding
\begin{equation}
	\mathcal{L}(y_t^\theta \mid s_t = s) = (\mathcal{T}^\pi p_\theta^\pi)(\cdot|s),
	\label{eq:single_step_law}
\end{equation}
where $\mathcal{L}(\cdot \mid s)$ denotes the conditional distribution given $s$. Thus, $y_t^\theta$ is a Monte Carlo target sample of $\mathcal{T}^\pi p_\theta^\pi$. Building upon this distributional TD sample, the final distributional return target $\hat{z}_t$ is constructed as follows:
\begin{equation}
	\hat{z}_t = \sum_{l=0}^{T-t-1} (\gamma\lambda)^l \zeta_{t+l} + \mathcal{F}_\theta(s_t, \varepsilon_t).
	\label{eq:dist_return_target}
\end{equation}
Defining the $n$-step distributional return $G_t^{(n)} := \sum_{l=0}^{n-1} \gamma^l r_{t+l} + \gamma^n \mathcal{F}_\theta(s_{t+n}, \varepsilon_{t+n})$, recursive substitution into Eq.~\eqref{eq:single_step_law} gives $\mathcal{L}(G_t^{(n)}|s_t) = ((\mathcal{T}^\pi)^n p_\theta^\pi)(\cdot|s_t)$. Equation~\eqref{eq:dist_return_target} can be rewritten as the $\lambda$-weighted sum of $n$-step returns~\cite{highbias}
\begin{equation}
	\hat{z}_t = (1-\lambda) \sum_{n=1}^{T-t-1} \lambda^{n-1} G_t^{(n)} + \lambda^{T-t-1} G_t^{(T-t)},
	\label{eq:lambda_weighted_target}
\end{equation}
whose conditional distribution defines the $\lambda$-weighted distributional Bellman operator induced by multi-step bootstrapping
\begin{equation}
	(\widetilde{\mathcal{T}}_\lambda^\pi p_\theta^\pi)(\cdot|s_t) := \mathcal{L}(\hat{z}_t \mid s_t),
	\label{eq:gae_bellman}
\end{equation}
where $\widetilde{\mathcal{T}}_\lambda^\pi$ generalizes $\mathcal{T}^\pi$ via GAE multi-step bootstrapping. From Eq.~\eqref{eq:gae_bellman}, $\hat{z}_t$ serves as a Monte Carlo target sample of $\widetilde{\mathcal{T}}_\lambda^\pi p_\theta^\pi$.

Then, let $\mathrm{sg}(\hat{z}_t)$ denote the stop-gradient version of the bootstrapped distributional target $\hat{z}_t$, so that no gradient is propagated through the target construction branch during optimization. Based on this target, we define a linear interpolation path $o^{t}$ between an independently sampled prior noise $\tilde{\varepsilon}$ and the target return
\begin{equation}
	o^{t} = (1 - t)\tilde{\varepsilon} + t \,\mathrm{sg}(\hat{z}_t), \quad t \in [0, 1],
	\label{eq:linear_interpolation}
\end{equation}
where $t$ parameterizes the interpolation, \textcolor{red}{and $\tilde{\varepsilon}$ is sampled from $\mathcal{N}(0,1)$ and is independent of $\{\varepsilon_t\}$ in Eq.~\eqref{eq:dist_return_target}.} 
Equation~\eqref{eq:linear_interpolation} defines a conditional path for each sampled pair $(\tilde{\varepsilon}, \mathrm{sg}(\hat{z}_t))$ rather than a global linear map between distributions. Conditional flow matching theory~\cite{flow,tong2024improving,liu2022flow} shows that the marginal flow induced by such conditional paths is not restricted to a linear transformation and can represent complex non Gaussian target distributions.

The instantaneous velocity field along this path is given by
\begin{equation}
	{u^t}(o^{t} \mid \mathrm{sg}(\hat{z}_t)) = \frac{d o^{t}}{dt} = \mathrm{sg}(\hat{z}_t) - \tilde{\varepsilon}.
	\label{eq:true_velocity}
\end{equation}
We incorporate the state \(s\) as an explicit condition into the velocity prediction network \(f_\theta(o^t, t, s)\), thereby implementing a classifier-free guidance mechanism. The training objective is defined as:
\begin{equation}
	\mathcal{L}_{\text{CFM}}(\theta) = \mathbb{E}_{s, \tilde{\varepsilon}, \hat{z}_t, t} \left[ \left\| f_\theta(o^t, t, s) - u^t(o^t \mid \mathrm{sg}(\hat{z}_t)) \right\|^2 \right].
	\label{eq:flow_matching_loss}
\end{equation}
Unlike off-policy methods that use an independent target network to reduce target drift on the replay buffer, FlowCritic follows the on-policy data flow of PPO~\cite{ppo}, where targets are constructed online from rollouts of the current policy at each iteration. Based on the correspondence established above, the two endpoint distributions of the conditional path in Eq.~\eqref{eq:linear_interpolation} are $p_{\text{prior}}$ at the $\tilde{\varepsilon}$ end and $\widetilde{\mathcal{T}}_\lambda^\pi p_\theta^\pi$ at the $\mathrm{sg}(\hat{z}_t)$ end, with the latter defined in Eq.~\eqref{eq:gae_bellman}. By conditional flow matching theory~\cite{flow,liu2022flow,tong2024improving}, minimizing Eq.~\eqref{eq:flow_matching_loss} trains $f_\theta$ to learn the marginal flow that transports $p_{\text{prior}}$ to $\widetilde{\mathcal{T}}_\lambda^\pi p_\theta^\pi$. Thus, minimizing Eq.~\eqref{eq:flow_matching_loss} corresponds to a stochastic approximation of the distributional Bellman operator $\widetilde{\mathcal{T}}_\lambda^\pi$. Unlike C51 and QR-DQN, which rely on discrete supports, FlowCritic learns a continuous velocity field that can model arbitrary return distributions.
\end{revblock}
\subsection{Robust Value Learning with Truncation and Clipping}

To ensure stable learning of the value distribution during bootstrapping, we introduce a clipping mechanism for the velocity field updates. This method constrains the change between the current velocity field network $f_\theta$, and a target network $f_{\theta_{\text{old}}}$. For a given state $s_t$, interpolation point $o^t$, and time $t$, their respective predictions are:
\begin{equation}
	v_t^{\text{new}} = f_\theta(s_t, o^{t}, t), 
	\qquad 
	v_t^{\text{old}} = f_{\theta_{\text{old}}}(s_t, o^{t}, t).
\end{equation}
The clipping mechanism then constrains the deviation of the new prediction from the old one to construct a regularized velocity $v_t^{\text{clipped}}$:
\begin{equation}
	v_t^{\text{clipped}} = v_t^{\text{old}} + \operatorname{clip}\!\big(v_t^{\text{new}} - v_t^{\text{old}},\, -\delta,\, \delta \big),
\end{equation}
where the threshold $\delta > 0$ controls the maximum allowable update size, thereby preventing disruptive changes to the learned velocity field.
Finally, the flow matching loss in Eq.~\eqref{eq:flow_matching_loss} can be rewritten with the clipping mechanism as
\begin{equation}
	\mathcal{L}_{\text{CFM}}(\theta)=\mathbb{E}_{s,\tilde{\varepsilon},\hat{z}_t,t}\left[\max\left(\left\|v_t^{\text{new}}-\hat{u}^t\right\|^2,\ \left\|v_t^{\text{clipped}}-\hat{u}^t\right\|^2\right)\right],
	\label{eq:flow_loss_clipped}
\end{equation}
\begin{revblock}
where $\hat{u}^t$ is a shorthand for $u^t(o^t \mid \mathrm{sg}(\hat{z}_t)) = \mathrm{sg}(\hat{z}_t) - \tilde{\varepsilon}$.
\end{revblock}

The policy improvement phase requires a stable estimate of the advantage function $\hat{A}_t$ to ensure a reliable update direction.
To obtain a low-variance value baseline, we estimate the conditional expectation $V(s) = \mathbb{E}[Z^\pi(s)]$ via Monte Carlo sampling. For a given state $s_t$, we draw $n$ independent samples $\{\varepsilon_i\}_{i=1}^n \sim \mathcal{N}(0,1)$ from the prior distribution and transform them through the learned flow to generate $n$ return estimates:
\begin{equation}
	z_{t,i} = \mathrm{sg}\!\left(\mathcal{F}_\theta(s_t, \varepsilon_i)\right), \quad i = 1, \dots, n,
\label{sing_e}
\end{equation}
where $z_{t,i}$ denotes the $i$-th return sample generated for state $s_t$, and $\mathrm{sg}(\cdot)$ is the stop-gradient operator.
\begin{revblock}
The value estimate is computed as the sample mean of the generated return samples:
\begin{equation}
	\hat{V}(s_t) = \frac{1}{n} \sum_{i=1}^n z_{t,i},
	\label{eq_vanilla_estimator}
\end{equation}
which is also the expectation of the empirical distribution induced by $\{z_{t,i}\}_{i=1}^n$.
To construct a conservative value target and limit the influence of extreme right-tail samples generated by the flow critic, we use a truncated estimator that discards the top $m$ of $n$ sorted samples ($z_{t,[i]}$). The truncated estimator is defined as:\end{revblock}
\begin{equation}
	\hat{V}_{\text{trunc}}(s_t) = \frac{1}{n - m} \sum_{i=1}^{n - m} z_{t,[i]}.
	\label{eq:truncated_estimator}
\end{equation}
\textcolor{red}{Conservative value targets are commonly used in RL to mitigate overestimation~\cite{hasselt2010double,td3}. Unlike methods that maintain multiple critic networks, such as TD3, FlowCritic obtains a conservative target without introducing additional networks by discarding the top-$m$ samples from the return samples generated by the critic. This truncation introduces a conservative bias and limits the influence of extreme right-tail samples on the value target. Note that this truncation is applied only when constructing the value estimation target $\hat{V}_{\text{trunc}}(s_t)$, while the observed rewards and high-return rollout trajectories are left unchanged in the advantage computation.}
Then, we compute the TD errors using $\hat{V}_{\text{trunc}}(s)$ as the baseline:
\begin{equation}
	\delta_t = r_t + \gamma \hat{V}_{\text{trunc}}(s_{t+1}) - \hat{V}_{\text{trunc}}(s_t).
\end{equation}
The advantage estimate $\hat{A}_t$ is then computed via GAE by aggregating these TD errors with exponential weighting:
\begin{equation}
	\hat{A}_t = \sum_{l=0}^{T-t-1} (\gamma \lambda)^l \delta_{t+l}.
	\label{eq:advantage}
\end{equation}

In FlowCritic, truncation and clipping work in a concerted effort to enhance training stability and robustness. Truncation mitigates the estimation bias from extreme high-return samples, and provides a natural form of pessimism early in optimization. Meanwhile, clipping constrains excessive updates when the deviation between new and old predictions grows too large, ensuring a stable learning process. {\color{c51red}The truncation size and the clipping threshold are typically determined empirically.}

\subsection{Weighted Policy Optimization via Value Distribution Modeling}
\label{wflow}
\begin{algorithm}[t]
	\caption{FlowCritic}
	\LinesNumbered 
	\label{alg:dfpo}
	\textbf{Require:} Discount factor $\gamma$; GAE parameter $\lambda$; PPO clipping coefficient $\epsilon$; velocity field clipping threshold $\delta$; CoV temperature $\alpha$; number of samples $n$; truncation size $m$; number of parallel environments $N_e$; rollout length $T$; number of epochs $K$; minibatch size $B$
	
	Initialize policy parameters $\theta_\pi$ and flow matching model parameters $\theta_f$\;
	
	\For{iteration $= 1, 2, \dots$}{
		
	Collect a fresh rollout buffer $\mathcal{D} = \{(s_t^{(i)}, a_t^{(i)}, r_t^{(i)}, s_{t+1}^{(i)})\}$ by executing the current policy in $N_e$ parallel environments for $T$ steps (Eq.~\eqref{eq:dataset})\;
			
		\For{each state $s_t$ in rollout buffer $\mathcal{D}$}{
			Construct distributional return target $\hat{z}_t$ (Eq.~\eqref{eq:dist_return_target})\;
			Sample $\{\varepsilon_i\}_{i=1}^n \sim \mathcal{N}(0,1)$ from the prior distribution\;
			Generate return samples $\{z_{t,i}\} = \{\mathcal{F}_{\theta_f}(s_t, \varepsilon_i)\}_{i=1}^n$ via flow matching\;
			Compute truncated value estimate $\hat V(s_t)$ (Eq.~\eqref{eq:truncated_estimator})\;
			Calculate advantage (Eq.~\eqref{eq:advantage})\;
			Compute state-dependent weight $w_t$ based on CoV (Eq.~\eqref{eq:weight_cv})\;
		}
			
		\For{epoch $k = 1, \dots, K$}{
			Shuffle rollout buffer $\mathcal{D}$ and split into minibatches of size $B$\;
			\For{each minibatch $\mathcal{B} \subset \mathcal{D}$}{
				Update flow matching model: minimize clipped flow matching loss (Eq.~\eqref{eq:flow_loss_clipped})\;
				$\theta_f \gets \theta_f - \eta_f \nabla_{\theta_f}\mathcal{L}_{\text{CFM}}(\theta_f; \mathcal{B})$\;
				
				Update policy: maximize weighted PPO objective (Eq.~\eqref{eq:wppo_obj})\;
				$\theta_\pi \gets \theta_\pi + \eta_\pi \nabla_{\theta_\pi} \mathcal{L}_{\mathrm{WPPO}}(\theta_\pi; \mathcal{B})$\;
			}
		}
	}
\end{algorithm}
To fully exploit the structural information within the value distribution, we adopt a signal processing perspective, where we treat the set of return samples $\{ z_{t,i} \}_{i=1}^n$ as a collection of noisy observations of the true value $V^\ast(s_t)$. A critical property of these observations is that their noise level is not constant; instead, it varies significantly across different states as a direct consequence of the environment's intrinsic stochasticity and the uncertainty inherent in trajectory sampling.
Consequently, uniformly treating samples with varying signal-to-noise ratios (SNRs) is suboptimal, as this approach allows high-noise samples to introduce excessive variance into the policy gradient, thereby degrading the algorithm’s performance. Therefore, we propose an adaptive weighting mechanism for policy optimization, as illustrated in Fig.~\ref{fig:weighting_scheme}.
\begin{figure}
	\centering
	\includegraphics[width=0.47\textwidth]{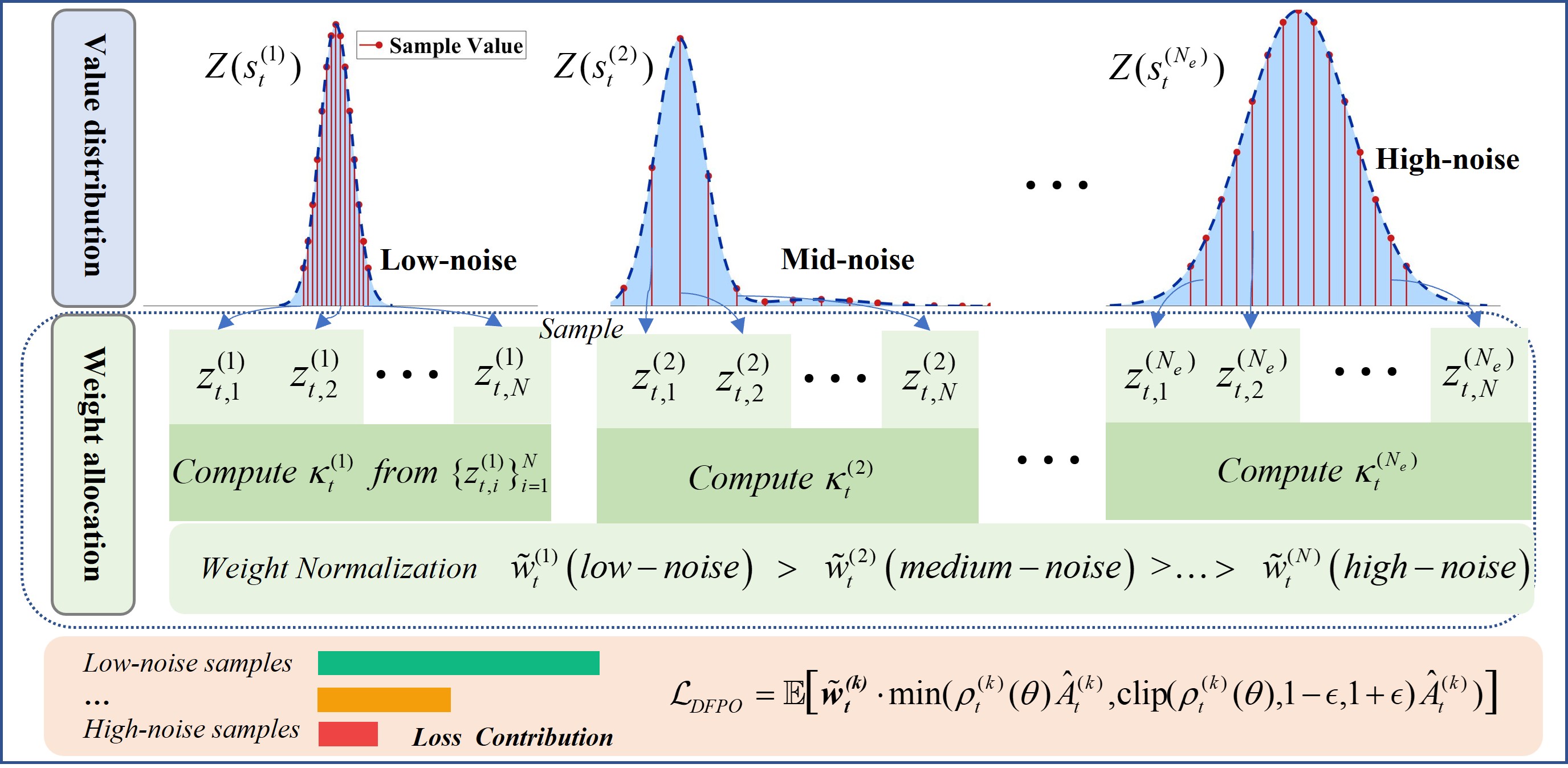}
	\caption{FlowCritic weighting mechanism. Value distributions with lower noise levels receive larger weights in policy optimization, ensuring that reliable estimates dominate the learning process.}
	\label{fig:weighting_scheme}
\end{figure}

Specifically, we generate rollouts by executing the current policy $\pi_\theta$ for $T$ timesteps across $N_e$ parallel environments. These trajectories are gathered into a temporary collection $\mathcal{D}$:
\begin{equation}
	\mathcal{D} =
	\Big\{ (s_t^{(k)}, a_t^{(k)}, r_t^{(k)}, s_{t+1}^{(k)}) 
	\;\Big|\;
	t = 0,\dots,T-1,\; k = 1,\dots,N_e \Big\},
	\label{eq:dataset}
\end{equation}
where \(k\) indexes the parallel environments, \(s_t^{(k)}\) denotes the state at step \(t\) in the \(k\)-th environment, and \(|\mathcal{D}| = N_e \times T\).
For each state \(s_t^{(k)}\) in $\mathcal{D}$, we then generate a set of \(n\) return samples by applying the sampling operation defined in Eq.~\eqref{sing_e}. This yields the following sample set for each state:
\begin{equation}
	\{ z_{t,i}^{(k)} \}_{i=1}^n = \{ \mathcal{F}_\theta(s_t^{(k)}, \varepsilon_i) \}_{i=1}^n,
	\label{eq:return_samples}
\end{equation}
where \(z_{t,i}^{(k)}\) denotes the \(i\)-th generated return for state \(s_t^{(k)}\).

To adaptively quantify the noise level, we introduce the coefficient of variation (CoV), denoted by \(\kappa_t^{(k)}\):
\begin{equation}
	\kappa_t^{(k)} = 
	\frac{\sigma_t^{(k)}}{|\mu_t^{(k)}| + \varepsilon},
	\qquad \varepsilon = 10^{-8},
	\label{eq:cv}
\end{equation}
{\color{red}
where \(\mu_t^{(k)} = \tfrac{1}{n}\sum_{i=1}^n {z_{t,i}^{(k)}}\) is the mean of the return samples for state \(s_t^{(k)}\), and 
\(\sigma_t^{(k)} = \sqrt{\tfrac{1}{n}\sum_{i=1}^n \big(z_{t,i}^{(k)} - \mu_t^{(k)}\big)^2}\) is their standard deviation. Compared with the standard deviation alone, the CoV normalizes noise by the scale of returns and thus avoids weighting distortions caused by differences in absolute magnitude. This property allows sample importance to be balanced adaptively across states. Although higher order statistics such as skewness and kurtosis describe distributional shape and may be useful for risk sensitive optimization, they do not reflect the reliability of an estimate itself. A distribution can be strongly skewed even when its samples remain tightly concentrated around the mean. By contrast, as the inverse of the SNR, the CoV provides a direct measure of estimation reliability and is therefore a principled choice for our noise aware weighting mechanism.}

Based on $\kappa_t^{(k)}$, the state-level sample weight is defined as:
\begin{equation}
	w_t^{(k)} = \exp\!\Big(- \alpha \cdot  \kappa_t^{(k)}\Big),
	\label{eq:weight_cv}
\end{equation}
where $\alpha>0$ is a hyperparameter. 
In practice, these weights are normalized across the batch: $\tilde{w}_t^{(k)} = \frac{w_t^{(k)}}{\sum_{j,l} w_l^{(j)}} \cdot N_e \cdot T$ to ensure $\sum_{k,t} \tilde{w}_t^{(k)} = N_e \cdot T$.
Finally, incorporating these normalized weights into the PPO objective yields the weighted loss function:
\begin{equation}
	\begin{aligned}
		\mathcal{L}_{\text{FlowCritic}}(\theta)
		&= \mathbb{E}_{(s_t^{(k)},a_t^{(k)})\sim\mathcal{D}} \Big[
		\tilde{w}_t^{(k)} \cdot
		\min\Big(
		\rho_t^{(k)}(\theta)\, \hat{A}_t^{(k)}, \\
		&\qquad\quad
		\operatorname{clip}\!\big(\rho_t^{(k)}(\theta),\, 1-\epsilon,\, 1+\epsilon\big)\, \hat{A}_t^{(k)}
		\Big)
		\Big],
		\label{eq:wppo_obj}
	\end{aligned}
\end{equation}
where \(\rho_t^{(k)}(\theta)=\tfrac{\pi_\theta(a_t^{(k)}\mid s_t^{(k)})}{\pi_{\theta_{\text{old}}}(a_t^{(k)}\mid s_t^{(k)})}\) is the importance sampling ratio, the advantage function \(\hat{A}_t^{(k)}\) is computed according to Eq.~\eqref{eq:advantage}.
The weight term $\tilde{w}_t^{(k)}$ in $\mathcal{L}_{\text{WPPO}}$ thus functions as an adaptive filter, explicitly suppressing the influence of high-noise samples on gradient estimation while amplifying high-confidence observations, thereby enhancing both the stability and convergence efficiency of policy updates. 

Based on the above formulations, the complete training procedure of FlowCritic is summarized in Algorithm~\ref{alg:dfpo}.

\section{Theoretical Analysis}
\label{the}
To further elucidate the properties of FlowCritic, we analyze its convergence  from a theoretical perspective, it can be shown that the distributional approximation process of FlowCritic is convergent, while the velocity field clipping mechanism provides additional stability during training. Before proceeding, we provide the following definitions and lemmas.

\begin{definition}\label{def:wasserstein}

	Let $\mu,\nu \in \mathcal{P}_p(\mathbb{R})$, where $\mathcal{P}_p(\mathbb{R})$ denotes the set of probability measures on $\mathbb{R}$ with finite $p$-th moment ($p \geq 1$). The $p$-Wasserstein distance between $\mu$ and $\nu$ is defined as
	\begin{equation}
		W_p(\mu,\nu) = \left( \inf_{\pi \in \Pi(\mu,\nu)} \mathbb{E}_{(X,Y)\sim\pi}[|X-Y|^p] \right)^{1/p},
	\end{equation}
	where $\Pi(\mu,\nu)$ denotes the set of all couplings of $\mu$ and $\nu$.  
	For two conditional distributions $\eta,\zeta \in \mathcal{M}$, we define the supremum metric as
	\begin{equation}
		\bar W_p(\eta,\zeta)=\sup_{s \in \mathcal{S}} W_p\!\big(\eta(\cdot \mid s),\zeta(\cdot \mid s)\big).
		\label{eq:sup-metric}
	\end{equation}
\end{definition}

\begin{definition}\label{def:approx-error}

	Let $\{p_{\theta_k}^\pi\}_{k=0}^\infty \subset \mathcal{M}$ denote the sequence of conditional return distributions generated by FlowCritic during training. The approximation error at iteration $k+1$ is defined as
	\begin{equation}
		\varepsilon_{k+1} = \bar W_p\!\left( p_{\theta_{k+1}}^\pi,\, \mathcal{T}^\pi p_{\theta_k}^\pi \right),
	\end{equation}
	where $\mathcal{T}^\pi$ denotes the distributional Bellman expectation operator as defined in Eq.~(\ref{bellman}).
\end{definition}
\begin{revblock}
\begin{lemma}\label{lemma:shift}
	Let $X\sim\mu$ and $Y\sim\nu$. For a constant $c\in\mathbb{R}$, define
	$X'=X+c$ and $Y'=Y+c$, and let $\mu'$ and $\nu'$ denote the distributions of $X'$ and $Y'$, respectively. Then
	\begin{equation}
		W_p(\mu',\nu') = W_p(\mu,\nu).
		\label{eq:lemma-shift}
	\end{equation}
	\end{lemma}

\begin{lemma}\label{lemma:homogeneity}
		Let $X\sim\mu$ and $Y\sim\nu$. For a constant $a\in\mathbb{R}$, define
		$X'=aX$ and $Y'=aY$, and let $\mu'$ and $\nu'$ denote the distributions of $X'$ and $Y'$, respectively. Then
		\begin{equation}
			W_p(\mu',\nu') = |a| W_p(\mu,\nu).
			\label{eq:lemma-homogeneity}
		\end{equation}
\end{lemma}
\end{revblock}
\begin{lemma}\label{lemma:triangle}

	For any $\mu,\nu,\lambda \in \mathcal{P}_p(\mathbb{R})$, it holds that
	\begin{equation}
		W_p(\mu,\lambda) \leq W_p(\mu,\nu) + W_p(\nu,\lambda).
	\end{equation}
	Furthermore, for any $\eta,\zeta,\rho \in \mathcal{M}$, we have
	\begin{equation}
		\bar W_p(\eta,\rho)\leq \bar W_p(\eta,\zeta)+\bar W_p(\zeta,\rho).
		\label{eq:triangle-metric}
	\end{equation}
\end{lemma}

\begin{theorem}[Convergence of Approximate Value Distribution Iteration]
	\label{thm:approx-conv}
	If the single-step approximation errors of FlowCritic satisfy $\varepsilon_{t+1} \leq \varepsilon_{\text{max}}$ for all $t \ge 0$, then the sequence of value distributions $\{p_{\theta_t}^\pi\}$ converges to a neighborhood of the true value distribution $p^\pi$:
	\[
	\limsup_{t \to \infty} \bar{W}_p\!\left(p_{\theta_t}^\pi,\, p^\pi\right) \leq \frac{\varepsilon_{\text{max}}}{1-\gamma}.
	\]
\end{theorem}
\begin{proof}
	Let $p_{\theta}^\pi, p_{\theta'}^\pi \in \mathcal{M}$ be two approximate value distributions generated by FlowCritic. By Eq.~\eqref{bellman}, $(\mathcal{T}^\pi p_{\theta}^\pi)(\cdot|s)$ and $(\mathcal{T}^\pi p_{\theta'}^\pi)(\cdot|s)$ represent the distributions of $r(s,a) + \gamma Z_\theta(s')$ and $r(s,a) + \gamma Z_{\theta'}(s')$, respectively, where $a \sim \pi(\cdot|s)$, $s' \sim P(\cdot|s,a)$, $Z_\theta(s') \sim p_{\theta}^\pi(\cdot|s')$, and $Z_{\theta'}(s') \sim p_{\theta'}^\pi(\cdot|s')$.
	
	We denote $P_{s,\pi}(a,s') = \pi(a|s)P(s'|s,a)$ as the joint distribution of $(a,s')$ given $s$. For each fixed $(a,s')$, let $\mu_{s,a,s'}^{(\theta)}$ and $\mu_{s,a,s'}^{(\theta')}$ be the conditional distributions of $r(s,a) + \gamma Z_\theta(s')$ and $r(s,a) + \gamma Z_{\theta'}(s')$, respectively. We then have:
	\begin{align}
		(\mathcal{T}^\pi p_{\theta}^\pi)(\cdot|s) &= \int_{(a,s')} \mu_{s,a,s'}^{(\theta)} dP_{s,\pi}(a,s'), \label{eq:mixture1}\\
		(\mathcal{T}^\pi p_{\theta'}^\pi)(\cdot|s) &= \int_{(a,s')} \mu_{s,a,s'}^{(\theta')} dP_{s,\pi}(a,s'). \label{eq:mixture2}
	\end{align}
	
	For a fixed $s \in \mathcal{S}$, let $\Pi_s$ denote all couplings $\xi$ with marginals $(\mathcal{T}^\pi p_{\theta}^\pi)(\cdot|s)$ and $(\mathcal{T}^\pi p_{\theta'}^\pi)(\cdot|s)$. Similarly, for each $(a,s')$, let $\Pi_{a,s'}$ denote all couplings $\xi_{a,s'}$ with marginals $\mu_{s,a,s'}^{(\theta)}$ and $\mu_{s,a,s'}^{(\theta')}$. Based on the coupling formulation of Wasserstein distance, we obtain:
	\begin{equation}
		W_p^p((\mathcal{T}^\pi p_{\theta}^\pi)(\cdot|s), (\mathcal{T}^\pi p_{\theta'}^\pi)(\cdot|s)) = \inf_{\xi \in \Pi_s} \int \int |x-y|^p d\xi(x,y). \label{eq:wass-mixture}
	\end{equation}
	Similarly, we have::
	\begin{equation}
		W_p^p(\mu_{s,a,s'}^{(\theta)}, \mu_{s,a,s'}^{(\theta')}) = \inf_{\xi_{a,s'} \in \Pi_{a,s'}} \int \int |x-y|^p d\xi_{a,s'}(x,y). \label{eq:wass-conditional}
	\end{equation}
	
	Let $\xi^*_{a,s'}$ be the optimal coupling achieving the infimum in Eq.~\eqref{eq:wass-conditional}. Since both mixtures share the same mixing measure $P_{s,\pi}$, we construct the coupling:
	\begin{equation}
		\hat{\xi} = \int_{(a,s')} \xi^*_{a,s'} dP_{s,\pi}(a,s'). \label{eq:coupling-construct}
	\end{equation}
	Since Wasserstein distance is the infimum over all couplings, it follows that:
	\begin{equation}
		W_p^p((\mathcal{T}^\pi p_{\theta}^\pi)(\cdot|s), (\mathcal{T}^\pi p_{\theta'}^\pi)(\cdot|s)) \leq \int \int |x-y|^p d\hat{\xi}(x,y). \label{eq:wass-upper}
	\end{equation}
	Expanding the right-hand side using the definition of $\hat{\xi}$, we get:
\begin{align}
	&\int \int |x-y|^p d\hat{\xi}(x,y) \nonumber\\
	&\qquad= \int_{(a,s')} \left( \int \int |x-y|^p d\xi^*_{a,s'}(x,y) \right) dP_{s,\pi}(a,s')\nonumber\\
	&\qquad= \int_{(a,s')} W_p^p(\mu_{s,a,s'}^{(\theta)}, \mu_{s,a,s'}^{(\theta')}) dP_{s,\pi}(a,s'), \label{eq:expand-coupling}
\end{align}
	where the second equality follows from Eq.~\eqref{eq:wass-conditional} since $\xi^*_{a,s'}$ is the optimal coupling.
	Combining Eqs.~\eqref{eq:wass-upper} and \eqref{eq:expand-coupling}, we arrive at:
	\begin{align}
		&W_p^p((\mathcal{T}^\pi p_{\theta}^\pi)(\cdot|s), (\mathcal{T}^\pi p_{\theta'}^\pi)(\cdot|s)) \nonumber\\
		&\qquad\leq \int_{(a,s')} W_p^p(\mu_{s,a,s'}^{(\theta)}, \mu_{s,a,s'}^{(\theta')}) dP_{s,\pi}(a,s'). \label{eq:wass-bound}
	\end{align}

Note that $\mu_{s,a,s'}^{(\theta)}$ and $\mu_{s,a,s'}^{(\theta')}$ are obtained via the affine transformation $z \mapsto r(s,a) + \gamma z$. By Lemma~\ref{lemma:homogeneity} and Lemma~\ref{lemma:shift}, we obtain:
\begin{equation}
	W_p^p(\mu_{s,a,s'}^{(\theta)}, \mu_{s,a,s'}^{(\theta')}) = \gamma^p W_p^p(p_{\theta}^\pi(\cdot|s'), p_{\theta'}^\pi(\cdot|s')). \label{eq:affine-wass}
\end{equation}
Substituting Eq.~\eqref{eq:affine-wass} into Eq.~\eqref{eq:wass-bound}, we find:
\begin{align}
	&W_p^p((\mathcal{T}^\pi p_{\theta}^\pi)(\cdot|s), (\mathcal{T}^\pi p_{\theta'}^\pi)(\cdot|s)) \nonumber\\
	&\qquad\leq \int_{(a,s')} W_p^p(\mu_{s,a,s'}^{(\theta)}, \mu_{s,a,s'}^{(\theta')}) dP_{s,\pi}(a,s') \nonumber\\
	&\qquad= \gamma^p \int_{(a,s')} W_p^p(p_{\theta}^\pi(\cdot|s'), p_{\theta'}^\pi(\cdot|s')) dP_{s,\pi}(a,s') \nonumber\\
	&\qquad\leq \gamma^p \bar{W}_p^p(p_{\theta}^\pi, p_{\theta'}^\pi), \label{eq:contract-bound}
\end{align}
where the last inequality follows from $W_p^p(p_{\theta}^\pi(\cdot|s'), p_{\theta'}^\pi(\cdot|s')) \leq \bar{W}_p^p(p_{\theta}^\pi, p_{\theta'}^\pi)$ and $\int_{(a,s')} dP_{s,\pi}(a,s') = 1$.

Taking the $1/p$-th power of Eq.~\eqref{eq:contract-bound}, we deduce:
\begin{equation}
	W_p((\mathcal{T}^\pi p_{\theta}^\pi)(\cdot|s), (\mathcal{T}^\pi p_{\theta'}^\pi)(\cdot|s)) \leq \gamma \bar{W}_p(p_{\theta}^\pi, p_{\theta'}^\pi). \label{eq:contraction-s}
\end{equation}
Since Eq.~\eqref{eq:contraction-s} holds for all $s \in \mathcal{S}$, we have the supremum over $s$:
\begin{align}
	\bar{W}_p(\mathcal{T}^\pi p_{\theta}^\pi, \mathcal{T}^\pi p_{\theta'}^\pi) &= \sup_{s \in \mathcal{S}} W_p((\mathcal{T}^\pi p_{\theta}^\pi)(\cdot|s), (\mathcal{T}^\pi p_{\theta'}^\pi)(\cdot|s)) \nonumber\\
	&\leq \gamma \bar{W}_p(p_{\theta}^\pi, p_{\theta'}^\pi). \label{eq:contraction}
\end{align}

Equation~\eqref{eq:contraction} establishes that $\mathcal{T}^\pi$ is a contraction mapping. By Banach's fixed point theorem, there exists a unique fixed point $p^\pi$ satisfying $\mathcal{T}^\pi p^\pi = p^\pi$.
Let $d_t = \bar{W}_p(p_{\theta_t}^\pi, p^\pi)$ denote the distance between the $t$-th iterate and the true value distribution. From Lemma~\ref{lemma:triangle}, we have:
\begin{align}
	d_{t+1} &= \bar{W}_p(p_{\theta_{t+1}}^\pi, p^\pi) \nonumber\\
	&\leq \bar{W}_p(p_{\theta_{t+1}}^\pi, \mathcal{T}^\pi p_{\theta_t}^\pi) + \bar{W}_p(\mathcal{T}^\pi p_{\theta_t}^\pi, p^\pi). \label{eq:triangle-dt}
\end{align}
The first term is equal to $\bar{W}_p(p_{\theta_{t+1}}^\pi, \mathcal{T}^\pi p_{\theta_t}^\pi) = \varepsilon_{t+1}$ by definition. For the second term, using the fixed point property $p^\pi = \mathcal{T}^\pi p^\pi$ and Eq.~\eqref{eq:contraction}, we obtain:
\begin{align}
	\bar{W}_p(\mathcal{T}^\pi p_{\theta_t}^\pi, p^\pi) &= \bar{W}_p(\mathcal{T}^\pi p_{\theta_t}^\pi, \mathcal{T}^\pi p^\pi)\nonumber\\
	&\leq \gamma \bar{W}_p(p_{\theta_t}^\pi, p^\pi) = \gamma d_t. \label{eq:contract-dt}
\end{align}
Combining Eqs.~\eqref{eq:triangle-dt} and \eqref{eq:contract-dt}, we conclude:
\begin{equation}
	d_{t+1} \leq \varepsilon_{t+1} + \gamma d_t. \label{eq:recursion}
\end{equation}
Since $\varepsilon_{t+1} \leq \varepsilon_{\text{max}}$ for all $t \geq 0$ by assumption, it follows that:
\begin{equation}
	d_{t+1} \leq \varepsilon_{\text{max}} + \gamma d_t. \label{eq:recursion-bound}
\end{equation}
Iterating Eq.~\eqref{eq:recursion-bound}, we derive:
\begin{align}
	d_t &\leq \varepsilon_{\text{max}} + \gamma d_{t-1}\nonumber\\
	&\leq \varepsilon_{\text{max}} + \gamma(\varepsilon_{\text{max}} + \gamma d_{t-2})\nonumber\\
	&= \varepsilon_{\text{max}} \sum_{j=0}^{t-1} \gamma^j + \gamma^t d_0. \label{eq:dt-bound}
\end{align}
Taking the limit superior of Eq.~\eqref{eq:dt-bound} as $t \to \infty$, we get:
\begin{equation}
	\limsup_{t \to \infty} d_t \leq \varepsilon_{\text{max}} \cdot \frac{1}{1-\gamma} + 0 = \frac{\varepsilon_{\text{max}}}{1-\gamma}, \label{eq:limit-sup}
\end{equation}
where we use $\gamma^t d_0 \to 0$ as $t \to \infty$ since $\gamma \in [0,1)$.

Therefore, we establish:
\begin{equation}
	\limsup_{t \to \infty} \bar{W}_p(p_{\theta_t}^\pi, p^\pi) \leq \frac{\varepsilon_{\text{max}}}{1-\gamma}. \label{eq:final-bound}
\end{equation}
\begin{revblock}
The same analysis extends to the $\lambda$-weighted Bellman operator $\widetilde{\mathcal{T}}_\lambda^\pi$ used in our algorithm. Since $\widetilde{\mathcal{T}}_\lambda^\pi$ can be viewed as a $\lambda$-weighted mixture of $\{(\mathcal{T}^\pi)^n\}_{n\geq 1}$, a standard mixture-coupling argument gives an analogous contraction bound with a modulus no larger than $\gamma$ and typically smaller.
Hence the proof is complete.
\end{revblock}
\end{proof}
\begin{revblock}
\begin{corollary}[Uniform Error Bound under Bounded Moment Condition]
	\label{cor:bounded-error}
	Let $R_{\max} = \sup_{s,a} |r(s,a)|$. Following the bounded-moment setting in the Wasserstein analysis of DRL~\cite{bellemare,QR-DQN}, assume there exists $C > 0$ such that for some $p \geq 1$,
	\[
	\sup_{k\geq 0}\sup_{s\in\mathcal{S}} \left(\mathbb{E}_{Z\sim p_{\theta_k}^\pi(\cdot|s)}|Z|^p\right)^{1/p} \leq C.
	\]
	Then $\varepsilon_{k+1} \leq R_{\max} + (1+\gamma)C = \varepsilon_{\max}$ for all $k \geq 0$.
	\end{corollary}
	\begin{proof}
		For any $s \in \mathcal{S}$, let $Z_{k+1}(s) \sim p_{\theta_{k+1}}^\pi(\cdot|s)$, and define the Bellman target according to Eq.~\eqref{bellman} as
		\begin{equation}
		Y_k(s) = r(s,a) + \gamma Z_k(s'),
		\label{eq:cor-bellman-target}
		\end{equation}
		with $Z_k(s')\sim p_{\theta_k}^\pi(\cdot|s')$.
		
		From the bounded moment condition, we have $\|Z_{k+1}(s)\|_{L_p} \leq C$. For the term involving the random next state $s'$, by the law of total expectation, we obtain
		\begin{equation}
		\begin{aligned}
		\mathbb{E}|Z_k(s')|^p &= \mathbb{E}_{s'}\!\left[\mathbb{E}\!\left(|Z_k(s')|^p \mid s'\right)\right] \\
		&\leq C^p,
		\end{aligned}
		\label{eq:cor-tower}
		\end{equation}
		which gives $\|Z_k(s')\|_{L_p} \leq C$. Combining $|r(s,a)| \leq R_{\max}$ with Minkowski's inequality, we have
		\begin{equation}
		\begin{aligned}
		\|Y_k(s)\|_{L_p} &\leq \|r(s,a)\|_{L_p} + \gamma\|Z_k(s')\|_{L_p} \\
		&\leq R_{\max} + \gamma C.
		\end{aligned}
		\label{eq:cor-Y-bound}
		\end{equation}
		By the coupling formulation of $W_p$, taking the joint distribution of $(Z_{k+1}(s), Y_k(s))$ as a valid coupling, we obtain
		\begin{equation}
		\begin{aligned}
		&W_p\!\left(p_{\theta_{k+1}}^\pi(\cdot|s),\ (\mathcal{T}^\pi p_{\theta_k}^\pi)(\cdot|s)\right) \\
		&\qquad\leq \left(\mathbb{E}|Z_{k+1}(s) - Y_k(s)|^p\right)^{1/p}.
		\end{aligned}
		\label{eq:cor-coupling}
		\end{equation}
		Applying Minkowski's inequality once more to the right-hand side, we have
		\begin{equation}
		\begin{aligned}
		\left(\mathbb{E}|Z_{k+1}(s) - Y_k(s)|^p\right)^{1/p} 
		&\leq \|Z_{k+1}(s)\|_{L_p} + \|Y_k(s)\|_{L_p} \\
		&\leq C + R_{\max} + \gamma C.
		\end{aligned}
		\label{eq:cor-minkowski}
		\end{equation}
		Taking the supremum over $s$, we derive
		\begin{equation}
		\begin{aligned}
		\varepsilon_{k+1} = \bar W_p\!\left(p_{\theta_{k+1}}^\pi,\ \mathcal{T}^\pi p_{\theta_k}^\pi\right) 
		&\leq R_{\max} + (1+\gamma)C \\
		&= \varepsilon_{\max}.
		\end{aligned}
		\label{eq:cor-final}
		\end{equation}
		Hence the proof is complete.
\end{proof}

\begin{theorem}[Surrogate Variance Reduction via CoV Weighting]
	\label{thm:variance}
	Consider the CoV-weighted policy gradient estimator of FlowCritic. Let $N=N_e \cdot T$ denote the total number of samples collected from $N_e$ parallel environments, each generating trajectories of length $T$. The weighted estimator is
	\[
	\hat{g}_w = \frac{1}{N}\sum_{k=1}^{N_e}\sum_{t=0}^{T-1} \tilde{w}_t^{(k)} \cdot \nabla_{\theta_\pi} \log \pi_{\theta_\pi}(a_t^{(k)}|s_t^{(k)}) \cdot \hat{A}_t^{(k)}.
	\]
	The samples are generated as described in Section~\ref{wflow}. Then there exists $\alpha > 0$ such that
	\[
	\widetilde{\mathrm{Var}}[\hat{g}_w] < \widetilde{\mathrm{Var}}[\hat{g}],
	\]
	where $\widetilde{\mathrm{Var}}[\cdot]$ denotes the surrogate variance that ignores within-trajectory cross-covariance terms, and $\hat{g}$ is the unweighted policy gradient estimator.
\end{theorem}
\end{revblock}
\begin{proof}
	Let $g_t^{(k)} = \nabla_{\theta_\pi} \log \pi_{\theta_\pi}(a_t^{(k)}|s_t^{(k)}) \cdot \hat{A}_t^{(k)}$. The unweighted estimator is $\hat{g} = \frac{1}{N}\sum_{k,t} g_t^{(k)}$, and its variance can be decomposed as
	\begin{align}
		\text{Var}[\hat{g}]
		= \frac{1}{N^2} \!\left[
		\sum_{k,t} \text{Var}[g_t^{(k)}]
		+ \sum_{\Omega}
		\text{Cov}\!\left(g_t^{(k)}, g_{t'}^{(k')}\right)
		\right]\!,
		\label{eq:var-decomposition}
	\end{align}
	where $\Omega$ denotes all ordered index tuples $(k,t,k',t')$ satisfying $(k,t)\neq(k',t')$.
	We focus on the diagonal surrogate variance, which isolates the variance component directly affected by CoV weighting:
	\begin{align}
		\widetilde{\text{Var}}[\hat{g}] 
		&= \frac{1}{N^2} \sum_{k,t} \|\nabla_{\theta_\pi} \log \pi_{\theta_\pi}(a_t^{(k)}|s_t^{(k)})\|^2 \cdot \text{Var}[\hat{A}_t^{(k)}], \label{eq:pg-var-2}
	\end{align}
	and we define $\widetilde{\text{Var}}[\hat{g}_w]$ analogously for the weighted estimator.	
	Since $\hat{V}(s_t^{(k)}) = \frac{1}{n} \sum_{i=1}^n z_{t,i}^{(k)}$ averages $n$ value samples that are conditionally i.i.d. given $s_t^{(k)}$ (i.i.d. noise passed through the same deterministic map $\mathcal{F}_{\theta}$),
	\begin{equation}
		\text{Var}[\hat{V}(s_t^{(k)})] = \frac{(\sigma_t^{(k)})^2}{n}, \label{eq:v-var-derive}
	\end{equation}
	where $(\sigma_t^{(k)})^2$ is the variance of the return distribution at $s_t^{(k)}$. By the CoV definition in Eq.~\eqref{eq:cv},
	\begin{equation}
		(\sigma_t^{(k)})^2 = (\kappa_t^{(k)})^2 \cdot (|\mu_t^{(k)}| + \varepsilon)^2, \label{eq:sigma-from-cv}
	\end{equation}
	which gives
	\begin{equation}
		\text{Var}[\hat{V}(s_t^{(k)})] = \frac{(\kappa_t^{(k)})^2 \cdot (|\mu_t^{(k)}| + \varepsilon)^2}{n}. \label{eq:v-var-final}
	\end{equation}
	Under the same diagonal approximation, the variance of the GAE-based advantage satisfies
	\begin{equation}
		\text{Var}[\hat{A}_t^{(k)}] \approx \sum_{l=0}^{T-t-1} (\gamma \lambda)^{2l} \cdot \frac{(\kappa_{t+l}^{(k)})^2 \cdot (|\mu_{t+l}^{(k)}| + \varepsilon)^2}{n}, \label{eq:a-var-final}
	\end{equation}
	whose leading term scales as
	\begin{equation}
		\text{Var}[\hat{A}_t^{(k)}] \propto (\kappa_t^{(k)})^2. \label{eq:var-prop-simple}
	\end{equation}
For notational simplicity, we relabel each state-time pair $(k,t)$ by a single index $i \in \{1,\ldots,N\}$, with $\kappa_i = \kappa_t^{(k)}$. To isolate the effect of CoV-based weighting, we normalize the gradient-norm and state-dependent scale factors, retaining only the CoV-dependent component. We have
\begin{align}
    \widetilde{\mathrm{Var}}[\hat{g}]
    &= \frac{1}{N^2}\sum_{i=1}^{N}\kappa_i^2, \label{eq:var-unweighted-simple}\\
    \widetilde{\mathrm{Var}}[\hat{g}_w]
    &= \frac{1}{N^2}\sum_{i=1}^{N}\tilde{w}_i^2\kappa_i^2
    = \frac{\sum_{i=1}^{N} e^{-2\alpha\kappa_i}\kappa_i^2}
    {\left(\sum_{j=1}^{N}e^{-\alpha\kappa_j}\right)^2}.
    \label{eq:var-weighted-simple}
\end{align}
	Define $\phi(\alpha) = \widetilde{\text{Var}}[\hat{g}_w] - \widetilde{\text{Var}}[\hat{g}]$. At $\alpha = 0$, the two estimators coincide, so $\phi(0) = 0$. Differentiating at $\alpha = 0$ gives
	\begin{align}
		\phi'(0) 
		&= \frac{d}{d\alpha}\!\left[\frac{\sum_i e^{-2\alpha \kappa_i} \kappa_i^2}{\left(\sum_j e^{-\alpha \kappa_j}\right)^2}\right]_{\alpha=0} \nonumber\\
		&= \frac{2}{N^3}\!\left(\sum_{i=1}^N \kappa_i^2 \sum_{j=1}^N \kappa_j - N\sum_{i=1}^N \kappa_i^3\right). \label{eq:derivative-final}
	\end{align}
	By Chebyshev's sum inequality, when $\{\kappa_i\}$ are not all equal, we have
	\begin{equation}
		\sum_{i=1}^N \kappa_i^2 \sum_{j=1}^N \kappa_j < N\sum_{i=1}^N \kappa_i^3, \label{eq:inequality}
	\end{equation}
	and therefore $\phi'(0) < 0$. Together with $\phi(0) = 0$, this implies that there exists $\alpha > 0$ such that $\phi(\alpha) < 0$, that is, $\widetilde{\text{Var}}[\hat{g}_w] < \widetilde{\text{Var}}[\hat{g}]$. This completes the proof.
\end{proof}
\begin{revblock}
\begin{remark}
	\label{rem:variance-scope}
	The gap between $\widetilde{\text{Var}}[\hat{g}]$ and $\text{Var}[\hat{g}]$ comes from the cross-covariance terms in Eq.~\eqref{eq:var-decomposition}, which can be decomposed as
	\begin{align}
	\sum_{\Omega} \text{Cov}\!\left(g_t^{(k)}, g_{t'}^{(k')}\right) 
	= \mathcal{C}_{\text{cross}} + \mathcal{C}_{\text{within}},
	\label{eq:scope-cov-split}
	\end{align}
	where $\mathcal{C}_{\text{cross}} = \sum_{k \neq k'} \sum_{t,t'} \text{Cov}(g_t^{(k)}, g_{t'}^{(k')})$ collects all cross-environment terms, and $\mathcal{C}_{\text{within}} = \sum_{k} \sum_{t \neq t'} \text{Cov}(g_t^{(k)}, g_{t'}^{(k)})$ collects all within-trajectory terms.
	Since parallel environments are rolled out independently under separate random seeds, $\mathcal{C}_{\text{cross}} = 0$, and the gap reduces to $\mathcal{C}_{\text{within}}$. For $\mathcal{C}_{\text{within}}$, GAE induces correlations between $\hat{A}_t$ and $\hat{A}_{t'}$ through shared future TD errors. The component caused by this GAE overlap is modulated by $(\gamma\lambda)^{|t-t'|}$, while the remaining component is governed by Markov dynamics and policy-induced correlations. By the Cauchy-Schwarz inequality,
	\begin{equation}
	\left|\text{Cov}\!\left(g_t^{(k)}, g_{t'}^{(k)}\right)\right|
	\le \sqrt{\text{Var}[g_t^{(k)}] \cdot \text{Var}[g_{t'}^{(k)}]},
	\label{eq:scope-cs}
	\end{equation}
	and using the proportionality $\text{Var}[\hat{A}_t] \propto \kappa_t^2$ established in the proof, the CoV-dependent component of $\mathcal{C}_{\text{within}}$ can be bounded by diagonal variance terms involving $\{\kappa_t^{(k)}\}$. These terms share the same CoV structure as the diagonal terms. Therefore, $\widetilde{\text{Var}}[\hat{g}]$ is a reasonable surrogate for the component of $\text{Var}[\hat{g}]$ directly modulated by CoV weighting.
	\end{remark}

{\color{c51red}The above results are derived under the stated assumptions, and relaxing them is left for future work.}
\end{revblock}
\section{Experiments}
To comprehensively evaluate the effectiveness of our proposed FlowCritic algorithm, we conduct a series of comparative experiments and ablation studies across diverse environments. These experiments are designed to answer the following questions: (1) In stochastic environments, what advantages does FlowCritic's modeling of the complete state-value distribution offer over PPO's point estimation method? (2) On complex robotic control tasks, does FlowCritic outperform various baseline methods? (3) How do the different parameters of FlowCritic affect its final performance? (4) Can policies trained by FlowCritic be effectively deployed on real-world physical systems?
\begin{table}[h]
	\centering
	\caption{Hyperparameter settings for the control tasks.}
	\label{tab:hyperparams}
	\small
	\setlength{\tabcolsep}{4pt}
	\begin{tabular*}{\columnwidth}{@{\extracolsep{\fill}} l l l}
		\toprule
		\textbf{Group} & \textbf{Parameter} & \textbf{Value} \\
		\midrule
		Shared & Number of parallel environments ($N_e$) & 1024 \\
		       & Rollout length ($T$) & 16 \\
		       & Epochs per iteration ($K$) & 4 \\
		       & Optimizer & Adam \\
		       & Learning Rate & 5e-4 \\
		       & Discount Factor ($\gamma$) & 0.99 \\
		       & GAE Parameter ($\lambda$) & 0.95 \\
		       & PPO Clip Ratio ($\epsilon$) & 0.2 \\
		       & Minibatches & 2 \\
		       & Minibatch Size & 8192 \\
		       & Advantage Normalization & True \\
		       & Actor Network Architecture & [256]$\times$2 \\
		       & Critic Network Architecture & [512]$\times$4 \\
		       & Actor Activation Function & Tanh \\
		       & Critic Activation Function & ReLU \\
		       & Actor-Critic grad norm & 1.0 \\
		\midrule
		PPO\_AVC & Number of critic networks & 5 \\
		\midrule
		PPO\_CVE & Number of critic networks & 5 \\
		\midrule
		PPO\_C51 & Number of atoms & 51 \\
		\midrule
		PPO\_QD & Number of quantiles & 51 \\
		        & Huber loss parameter ($\kappa$) & 1.0 \\
		\midrule
		PPO\_TQC & Number of critic networks & 5 \\
		         & Number of quantiles per critic & 51 \\
		         & Number of dropped quantiles & 2 \\
		         & Huber loss parameter ($\kappa$) & 1.0 \\
		\midrule
		PPO\_IQN & Number of critic networks & 5 \\
		         & Number of cosine basis functions & 64 \\
		         & Number of training $\tau$ samples & 32 \\
		         & Number of evaluation $\tau$ samples & 64 \\
		\midrule
		FlowCritic & Number of generated value estimates ($n$) & 10 \\
		           & Truncated sample size ($m$) & 1 \\
		           & Velocity field clipping threshold ($\delta$) & 0.2 \\
		           & CoV temperature ($\alpha$) & 0.1 \\
		\bottomrule
	\end{tabular*}
\end{table}

\begin{table}[h]
	\centering
	\caption{\color{red}Dimensions of control tasks}
	\label{tab:env_details}
	\begin{tabular*}{\columnwidth}{@{\extracolsep{\fill}}lcc}
		\toprule
		\textbf{Environment (Abbr.)} & \textbf{Obs. Dim.} & \textbf{Act. Dim.} \\
		\midrule
		FrankaCubeStack (FCSk) & 19 & 7 \\
		Quadcopter (Quad) & 21 & 12 \\
		FrankaCabinet (FCab) & 23 & 9 \\
		BallBalance (Ball) & 24 & 3 \\
		ShadowHandOpenAIFF (SHFF) & 42 & 20 \\
		Anymal (Anym) & 48 & 12 \\
		Ant (Ant) & 60 & 8 \\
		AllegroHand (AHan) & 88 & 16 \\
		Humanoid (Huma) & 108 & 21 \\
		AllegroKuka (AKuk) & 117 & 23 \\
		AllegroKukaTwoArms (AK2A) & 196 & 46 \\
		ShadowHand (SHad) & 211 & 20 \\
		\bottomrule
	\end{tabular*}
	\vspace{-1em}
\end{table}
\label{exp}
\begin{figure*}[t]
	\centering
	\begin{subfigure}[b]{0.32\textwidth}
		\centering
		\includegraphics[width=\textwidth]{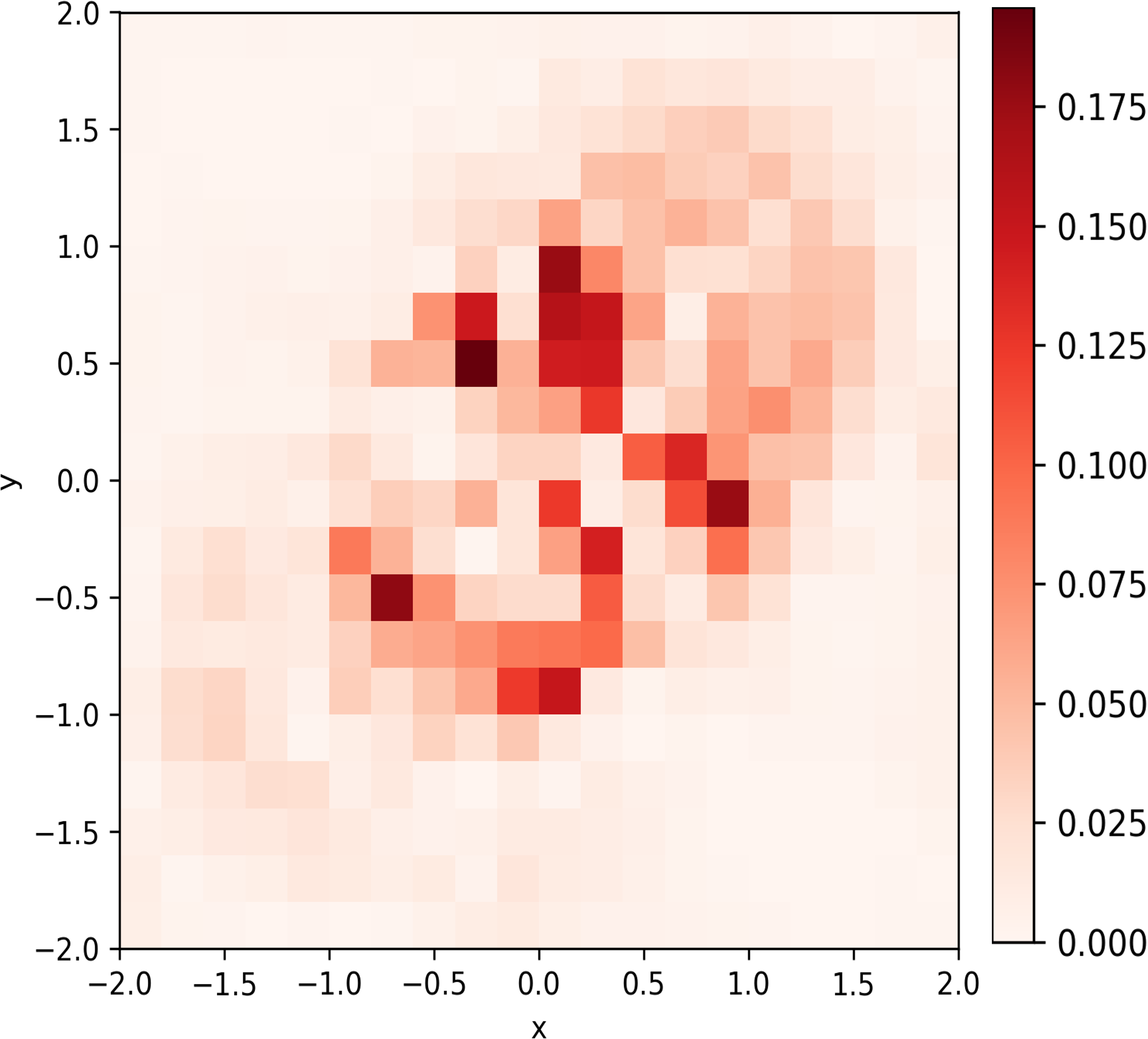}
		\caption{Absolute error of point estimation}
		\label{fig:toy2}
	\end{subfigure}
	\hfill
	\begin{subfigure}[b]{0.32\textwidth}
		\centering
		\includegraphics[width=\textwidth]{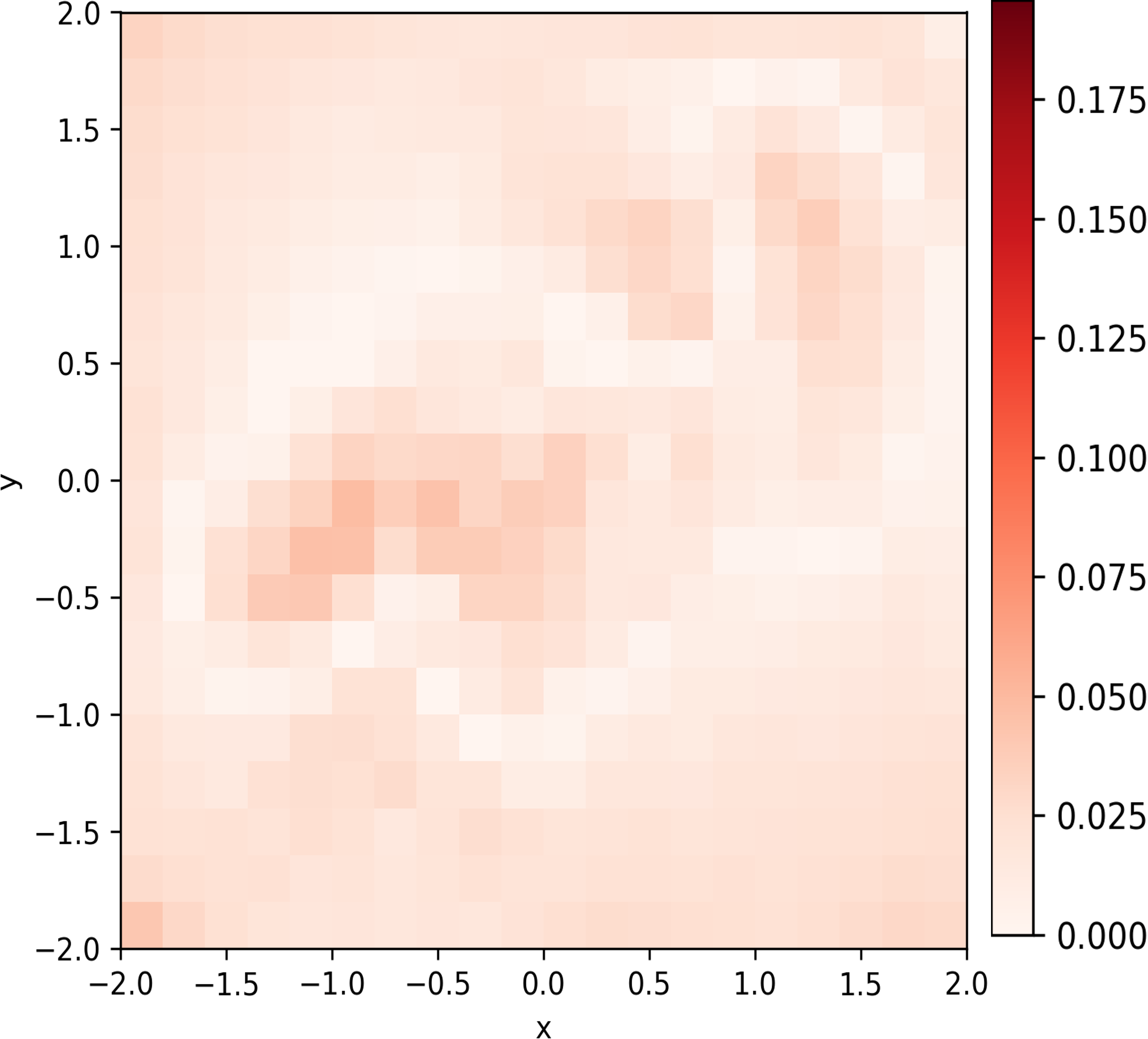}
		\caption{Absolute error of FlowCritic}
		\label{fig:toy3}
	\end{subfigure}
	\hfill
	\begin{subfigure}[b]{0.32\textwidth}
		\centering
		\includegraphics[width=\textwidth]{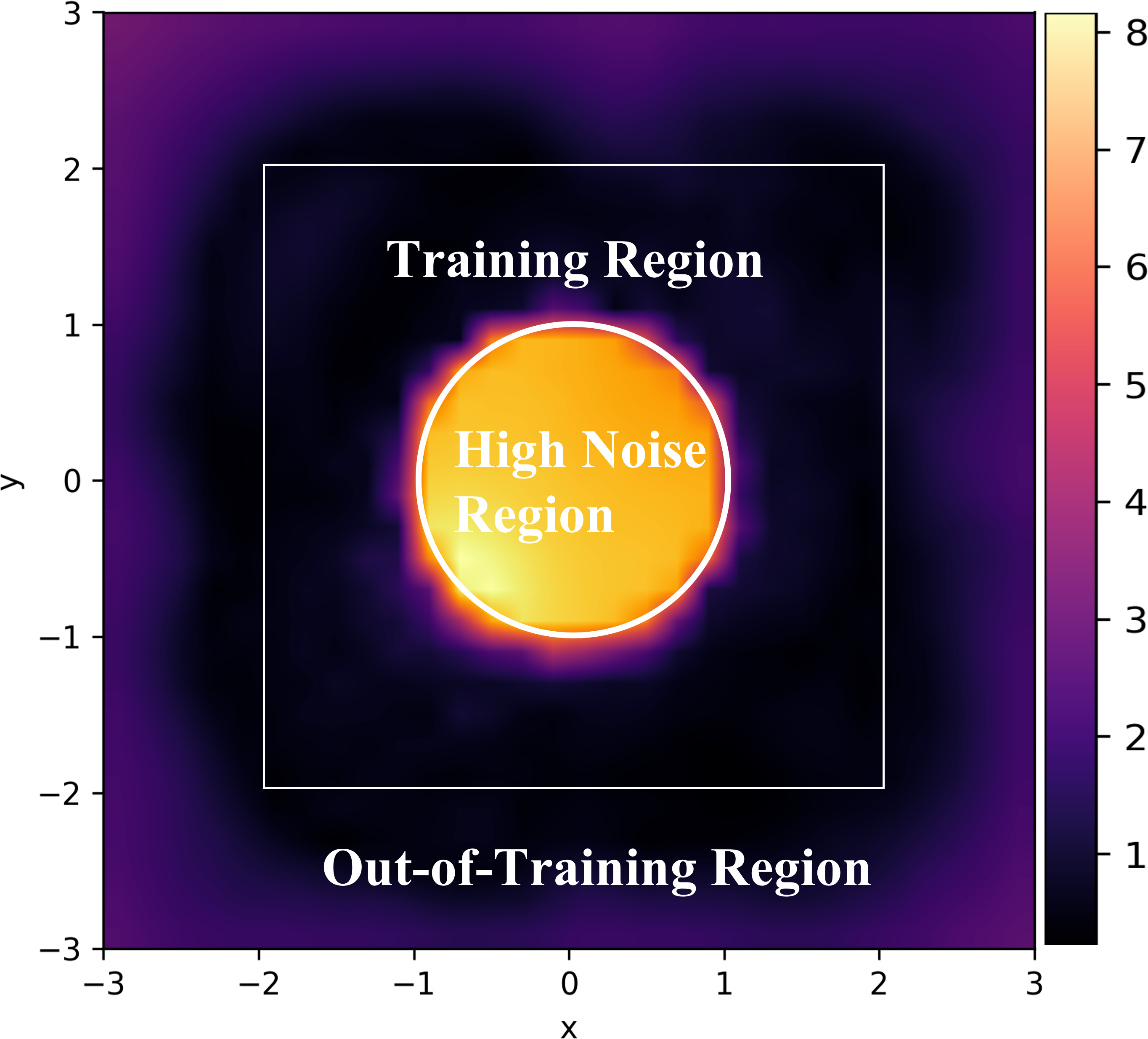}
		\caption{CoV estimation of FlowCritic}
		\label{fig:toycv}
	\end{subfigure}
	\caption{Results on the single-step environment.}
	\label{fig:toy}
	\vspace{-1em}
\end{figure*}

\begin{figure*}[t]
	\centering
	
	\begin{subfigure}[b]{0.32\textwidth}
		\centering
		\includegraphics[width=\textwidth]{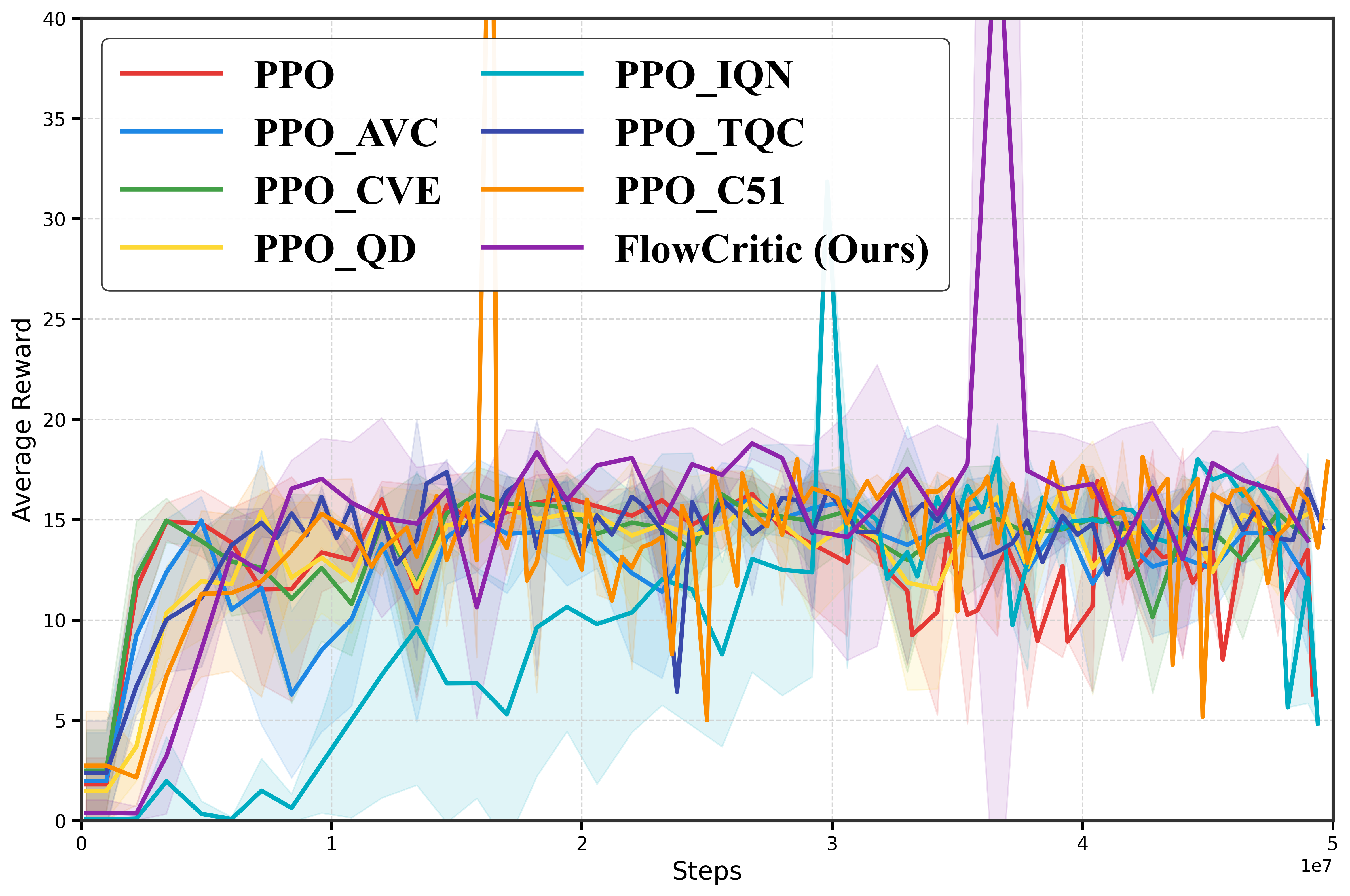}
		\caption{FrankaCubeStack}
		\label{fig:c51_FrankaCubeStack}
	\end{subfigure}
	\hfill
	\begin{subfigure}[b]{0.32\textwidth}
		\centering
		\includegraphics[width=\textwidth]{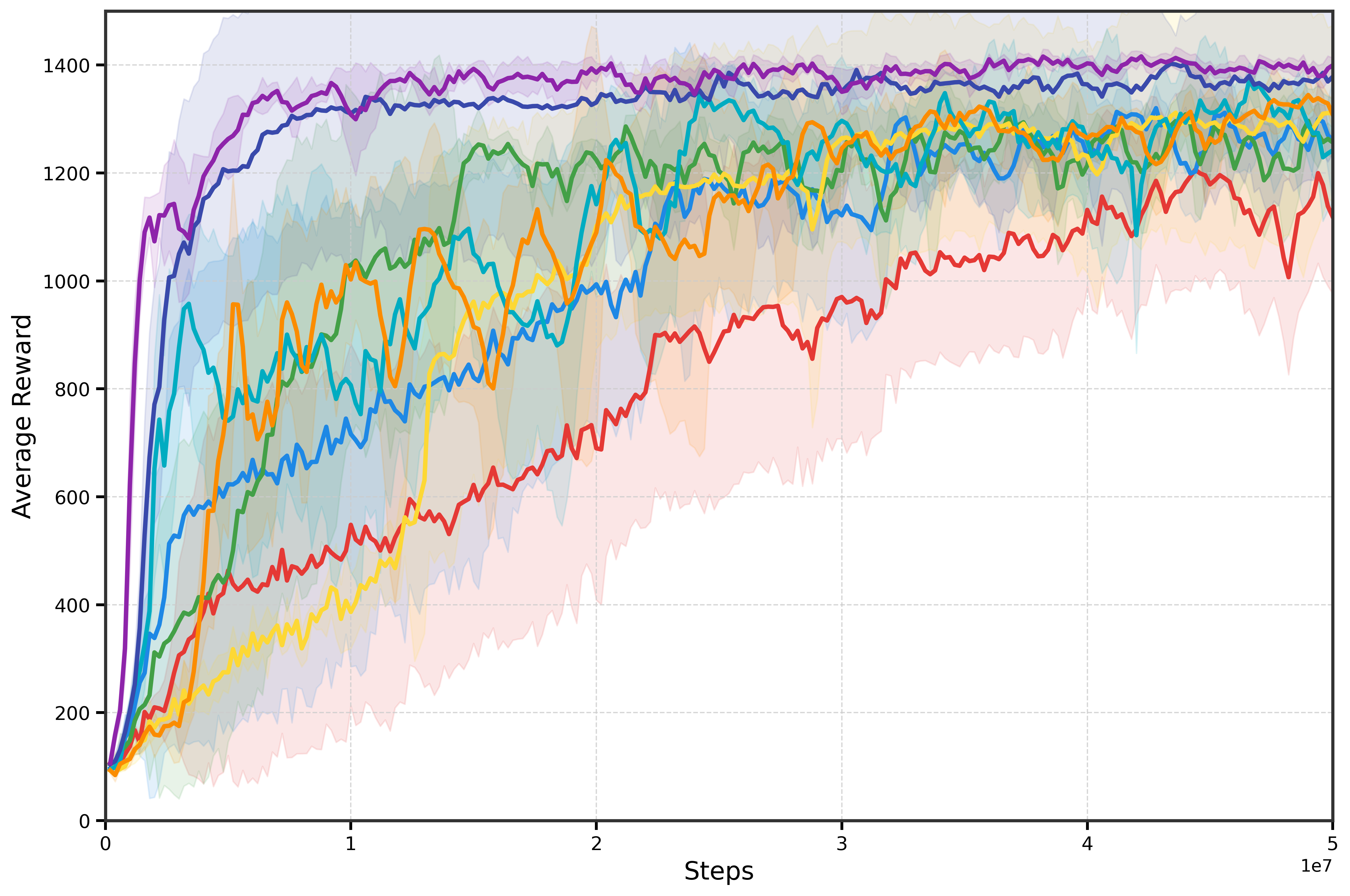}
		\caption{Quadcopter}
		\label{fig:c51_Quadcopter}
	\end{subfigure}
	\hfill
	\begin{subfigure}[b]{0.32\textwidth}
		\centering
		\includegraphics[width=\textwidth]{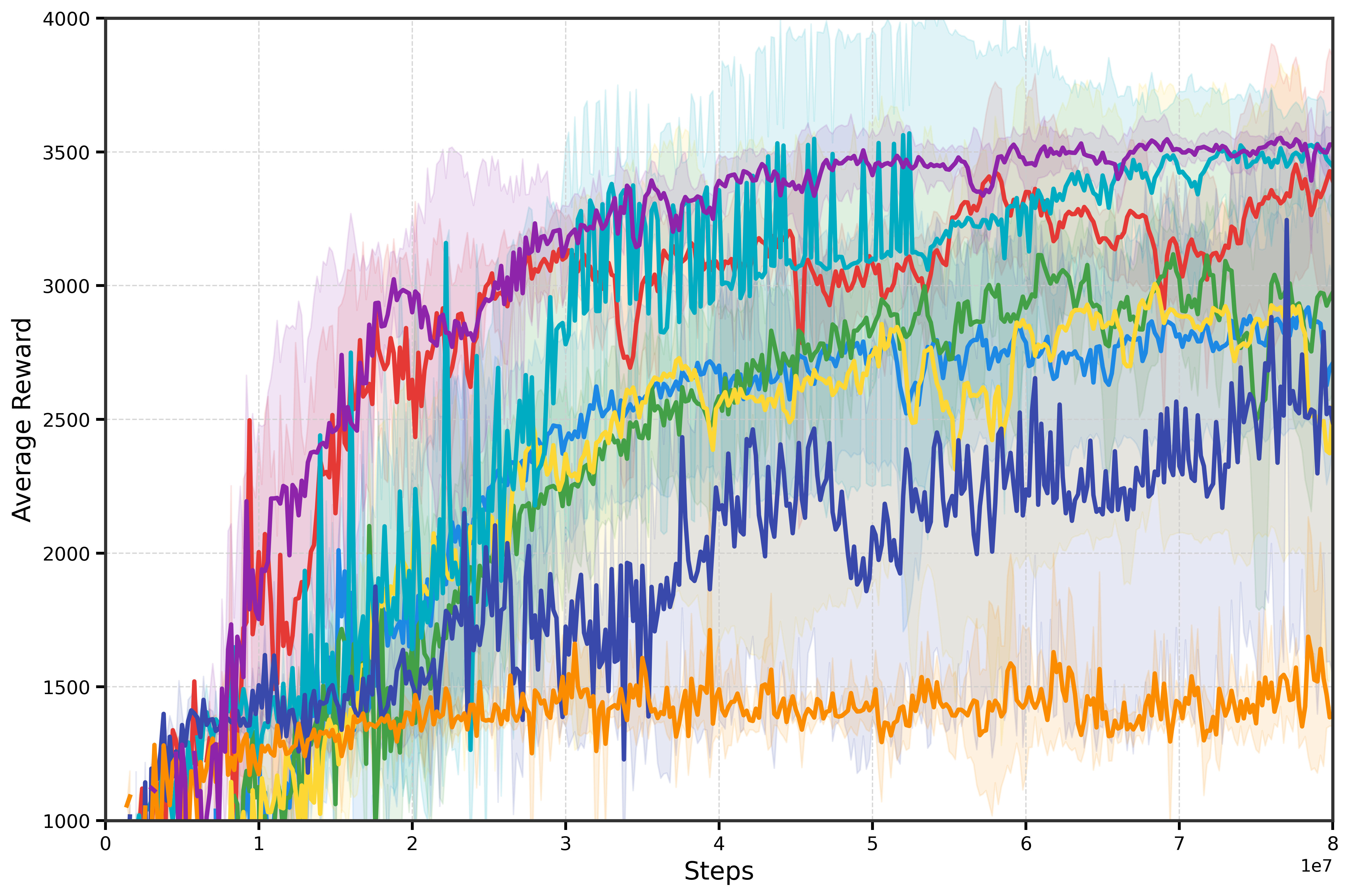}
		\caption{FrankaCabinet}
		\label{fig:c51_FrankaCabinet}
	\end{subfigure}
	\\ 
	
	\begin{subfigure}[b]{0.32\textwidth}
		\centering
		\includegraphics[width=\textwidth]{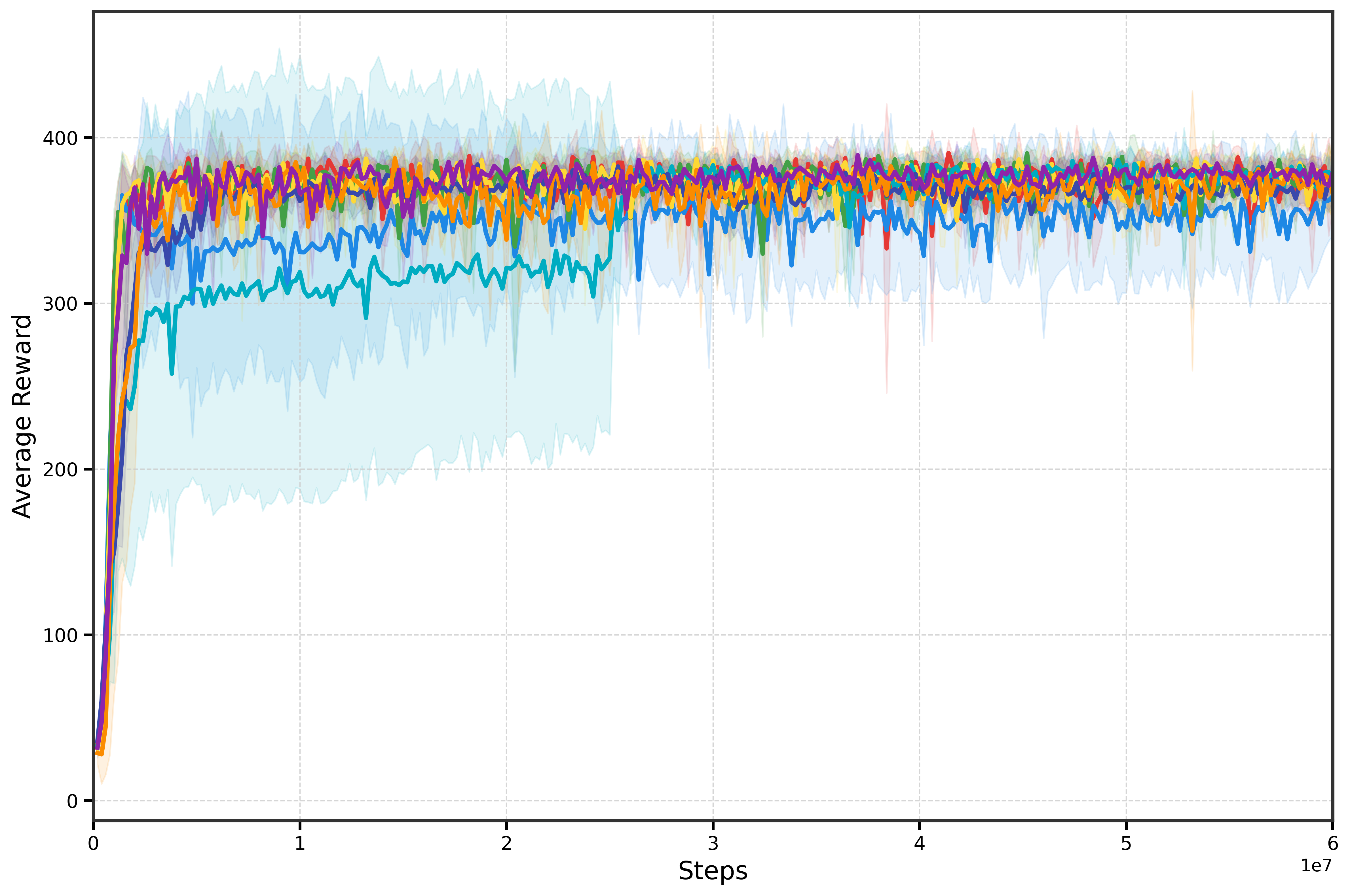}
		\caption{BallBalance}
		\label{fig:c51_BallBalance}
	\end{subfigure}
	\hfill
	\begin{subfigure}[b]{0.32\textwidth}
		\centering
		\includegraphics[width=\textwidth]{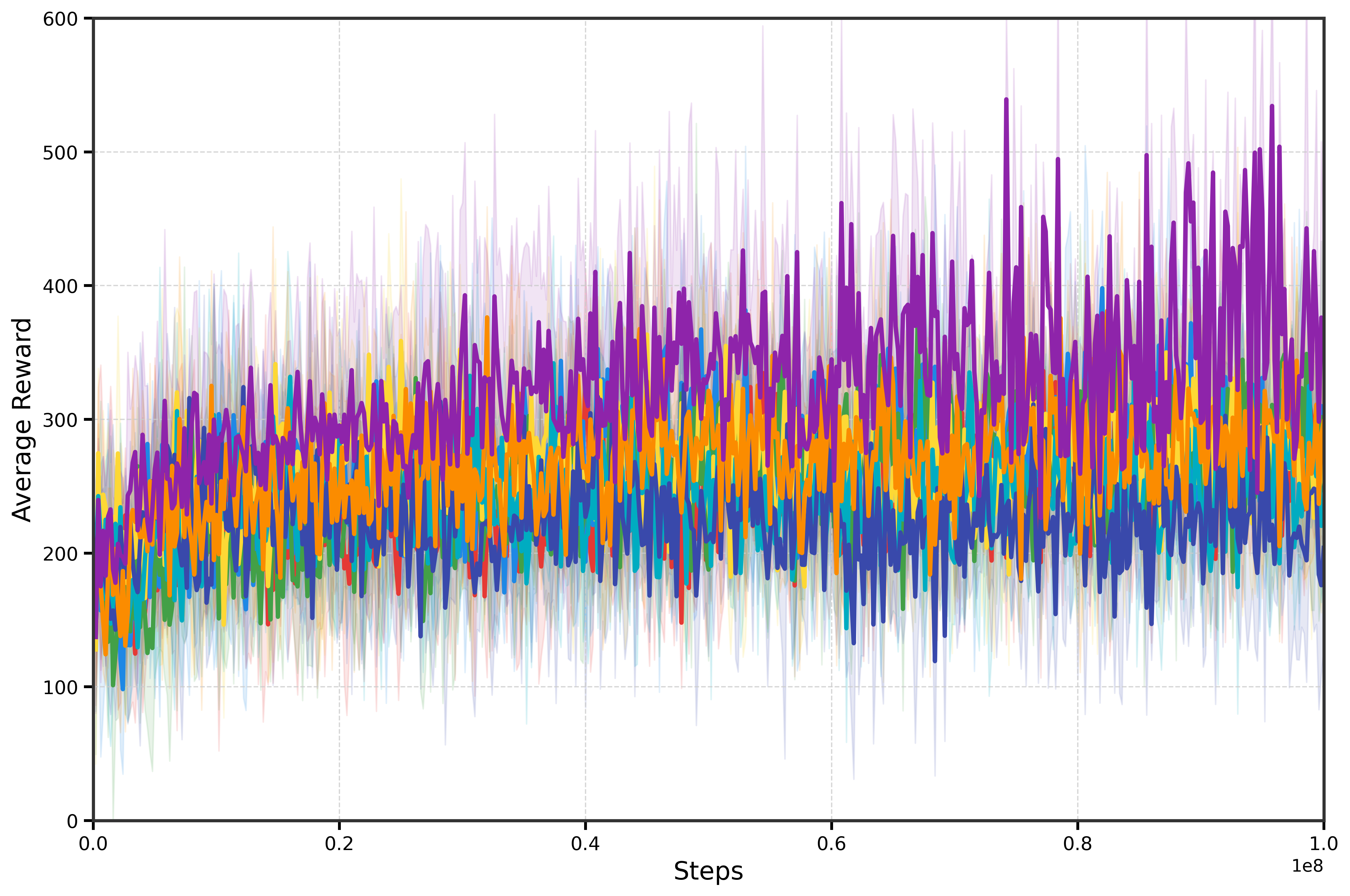}
		\caption{ShadowHandOpenAIFF}
		\label{fig:c51_ShadowHandOpenAIFF}
	\end{subfigure}
	\hfill
	\begin{subfigure}[b]{0.32\textwidth}
		\centering
		\includegraphics[width=\textwidth]{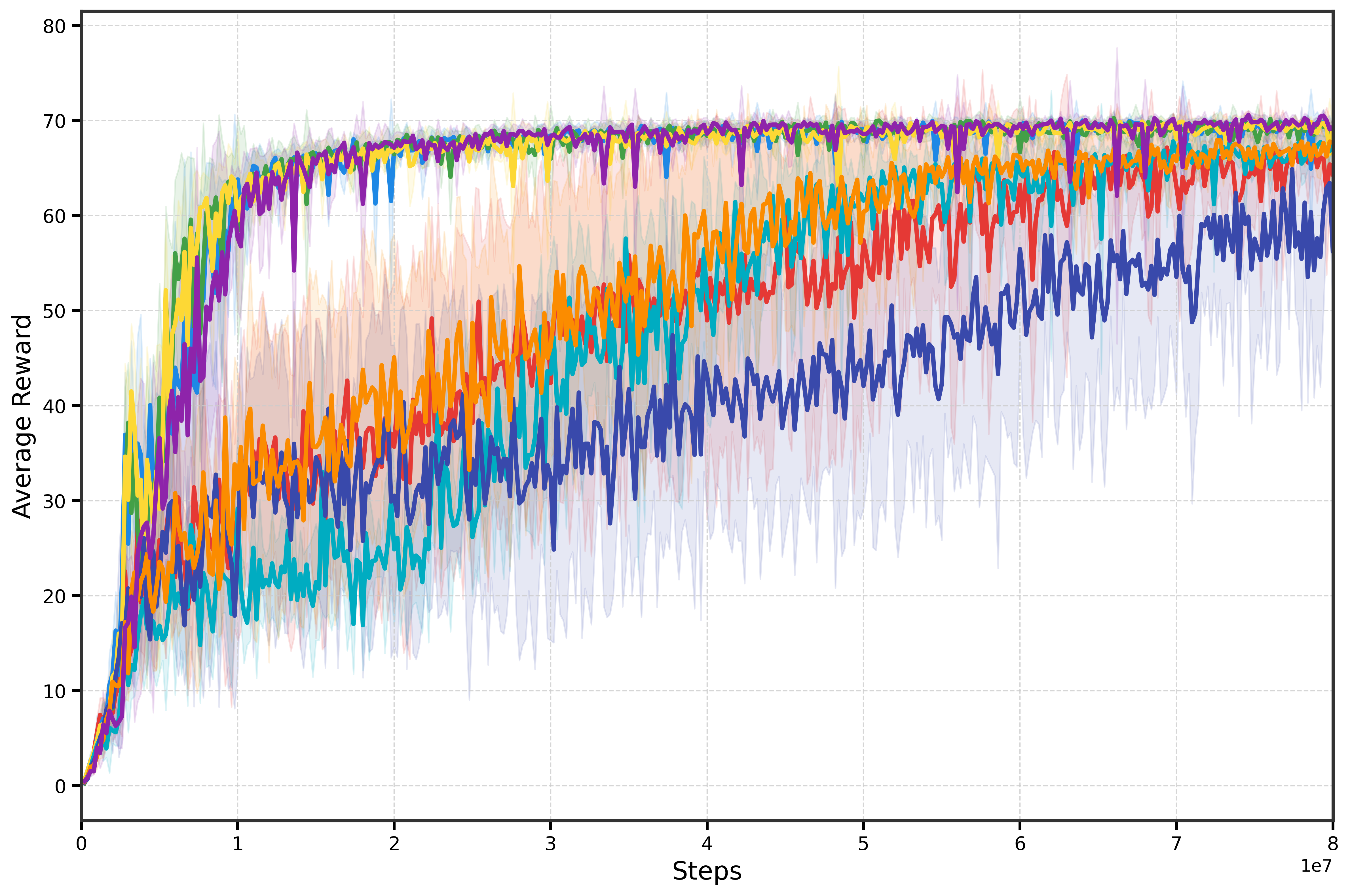}
		\caption{Anymal}
		\label{fig:c51_Anymal}
	\end{subfigure}
	\\ 
	
	\begin{subfigure}[b]{0.32\textwidth}
		\centering
		\includegraphics[width=\textwidth]{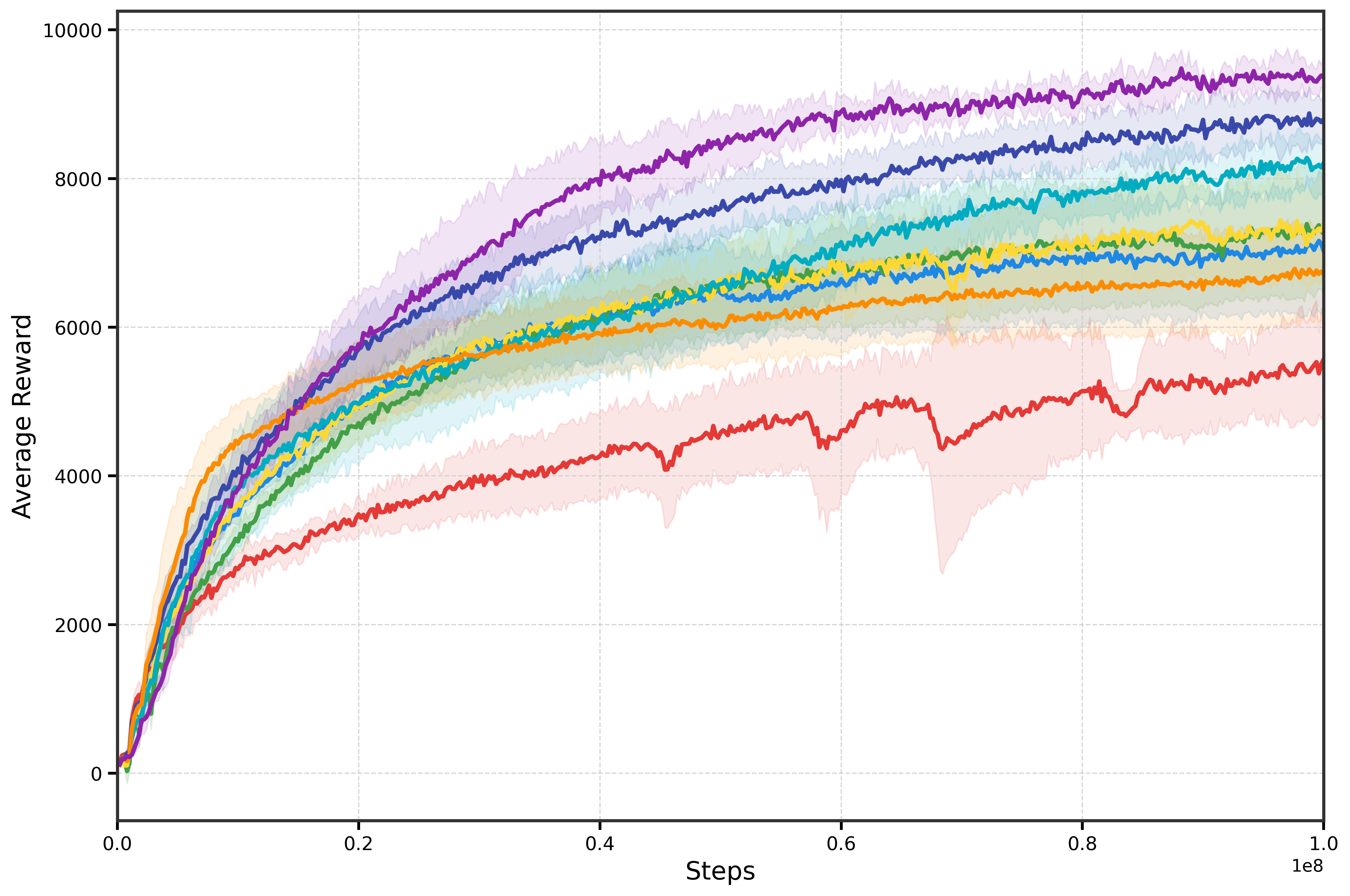}
		\caption{Ant}
		\label{fig:c51_Ant}
	\end{subfigure}
	\hfill
	\begin{subfigure}[b]{0.32\textwidth}
		\centering
		\includegraphics[width=\textwidth]{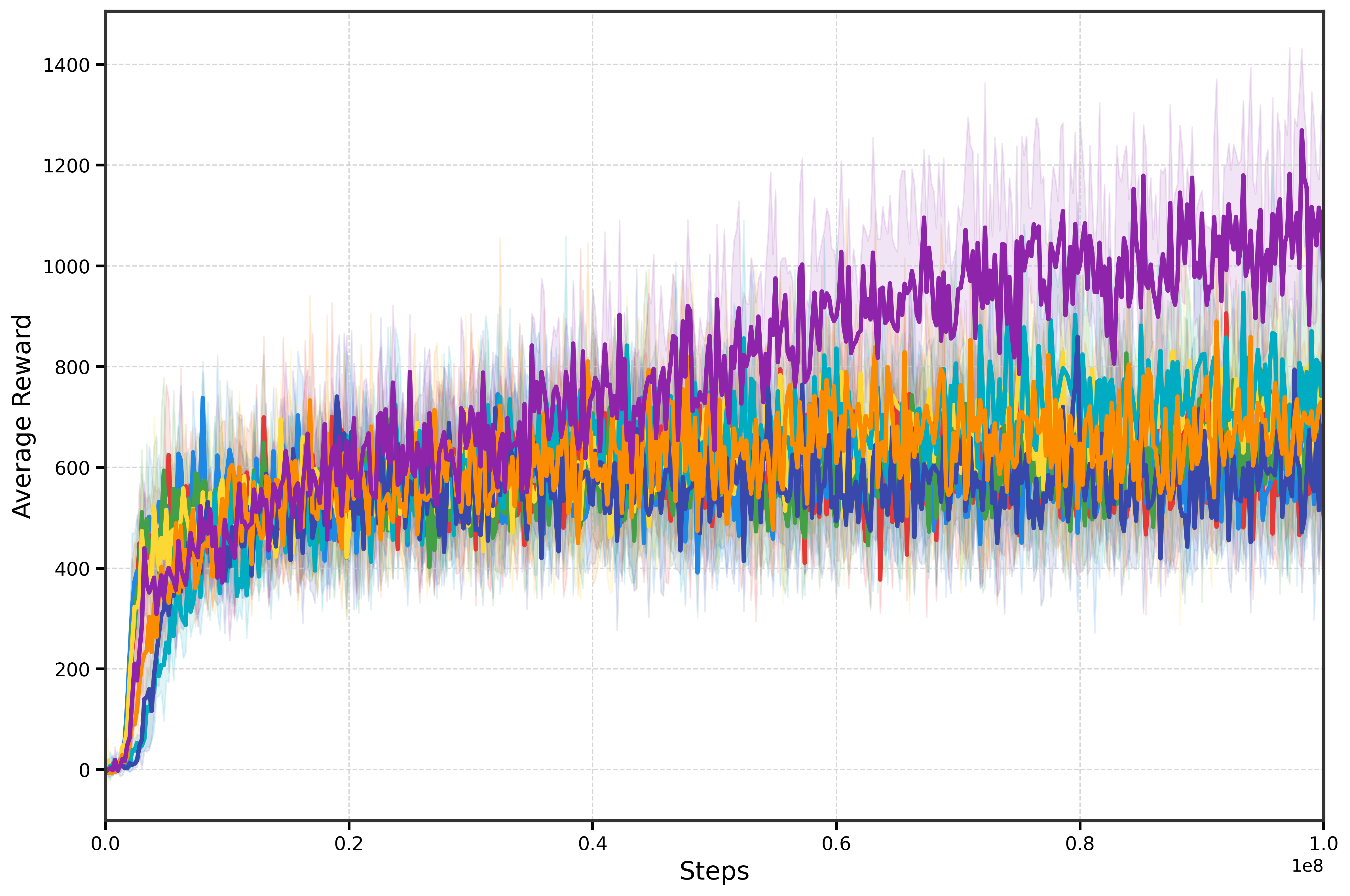}
		\caption{AllegroHand}
		\label{fig:c51_AllegroHand}
	\end{subfigure}
	\hfill
	\begin{subfigure}[b]{0.32\textwidth}
		\centering
		\includegraphics[width=\textwidth]{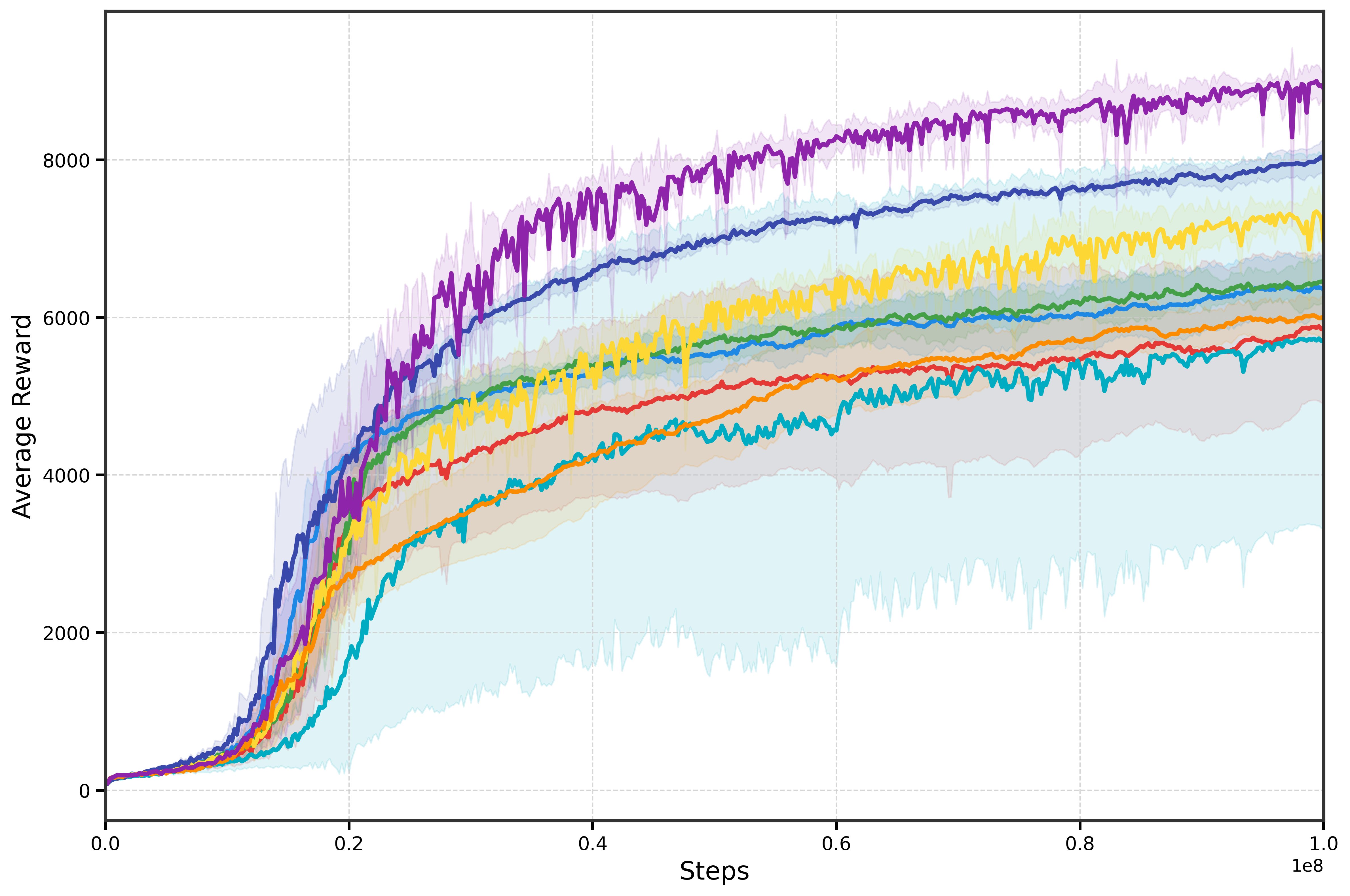}
		\caption{Humanoid}
		\label{fig:c51_Humanoid}
	\end{subfigure}
	\\ 
	
	\begin{subfigure}[b]{0.32\textwidth}
		\centering
		\includegraphics[width=\textwidth]{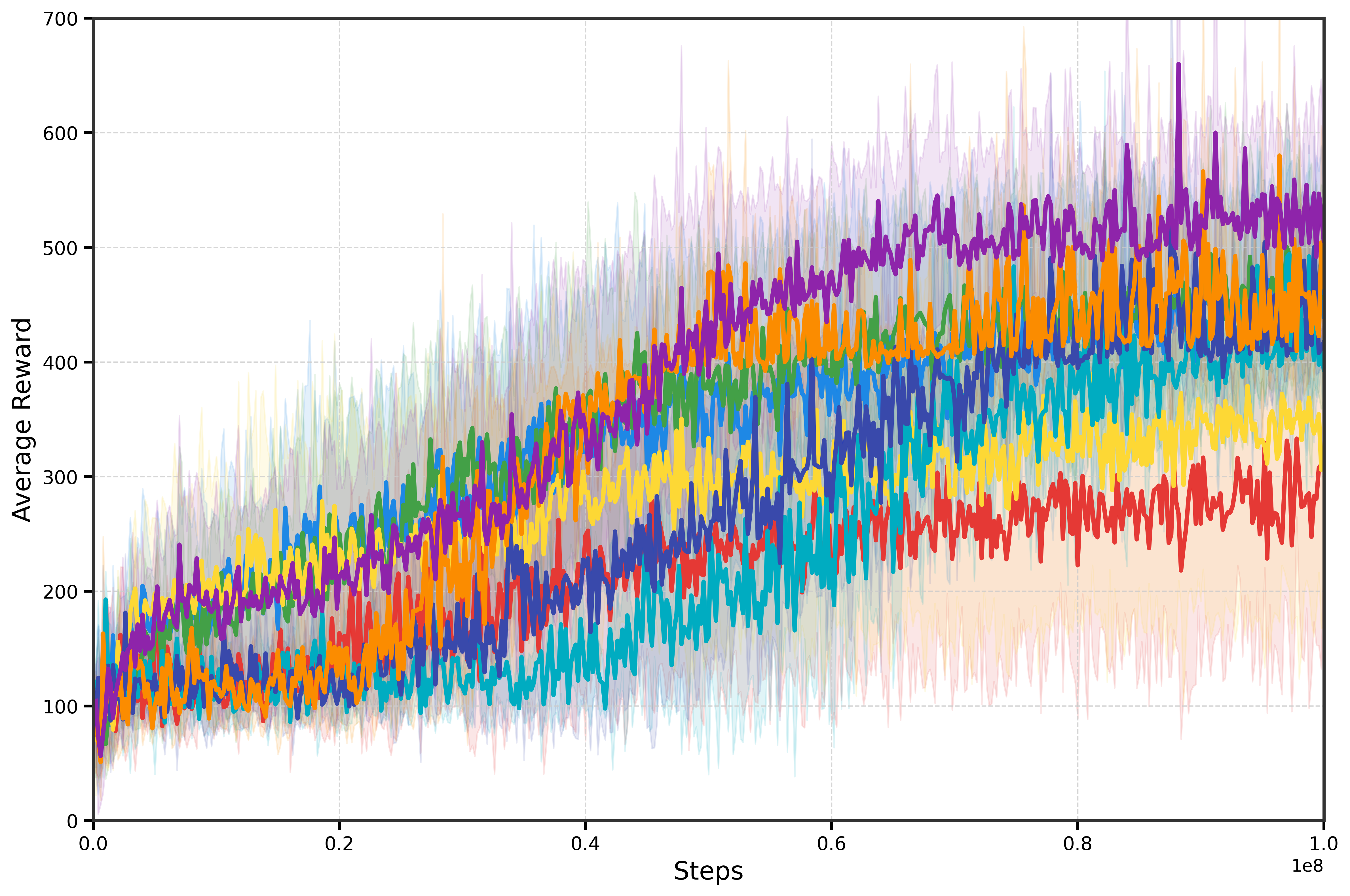}
		\caption{AllegroKuka}
		\label{fig:c51_AllegroKuka}
	\end{subfigure}
	\hfill
	\begin{subfigure}[b]{0.32\textwidth}
		\centering
		\includegraphics[width=\textwidth]{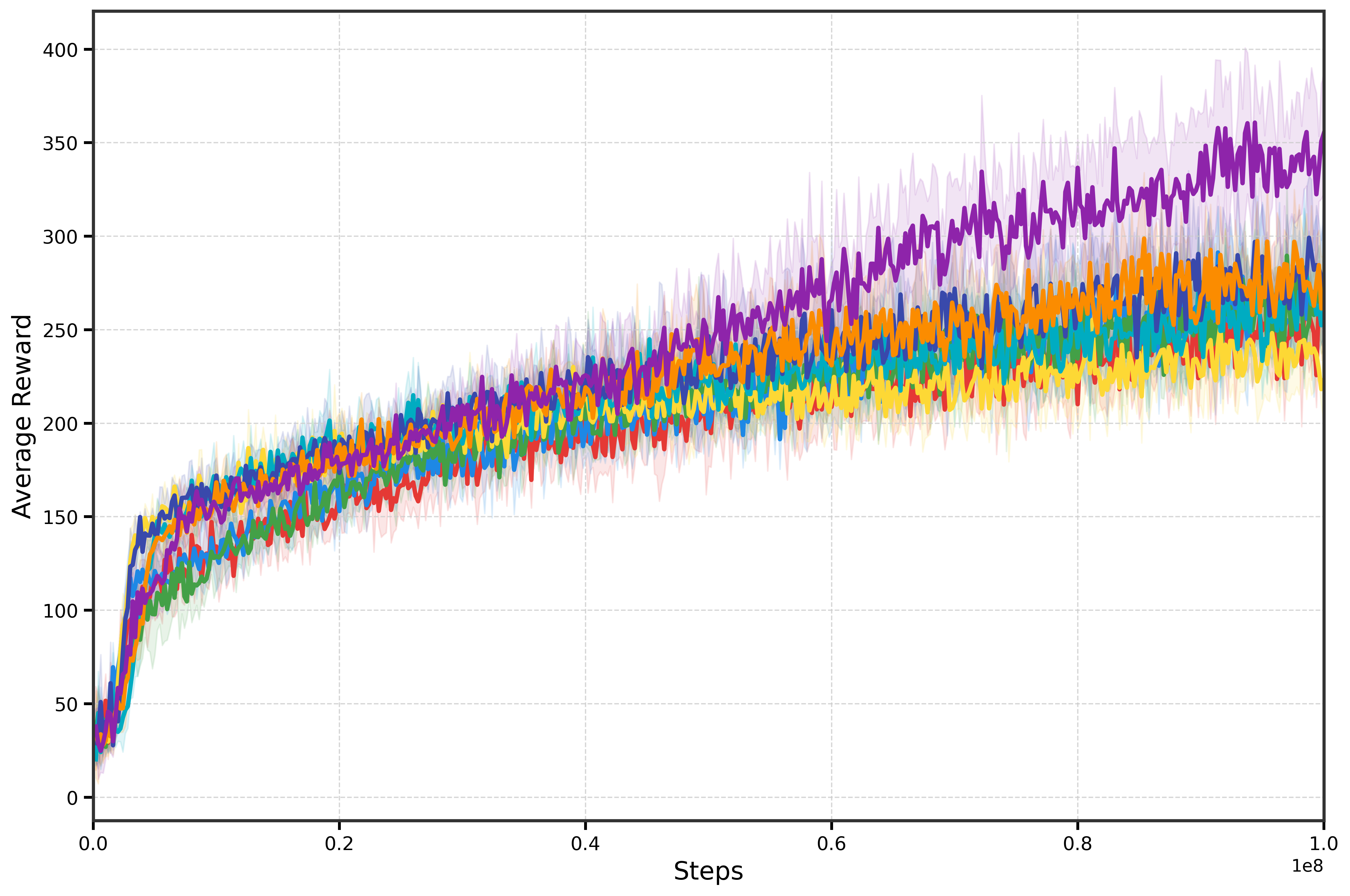}
		\caption{AllegroKukaTwoArms}
		\label{fig:c51_AllegroKukaTwoArms}
	\end{subfigure}
	\hfill
	\begin{subfigure}[b]{0.32\textwidth}
		\centering
		\includegraphics[width=\textwidth]{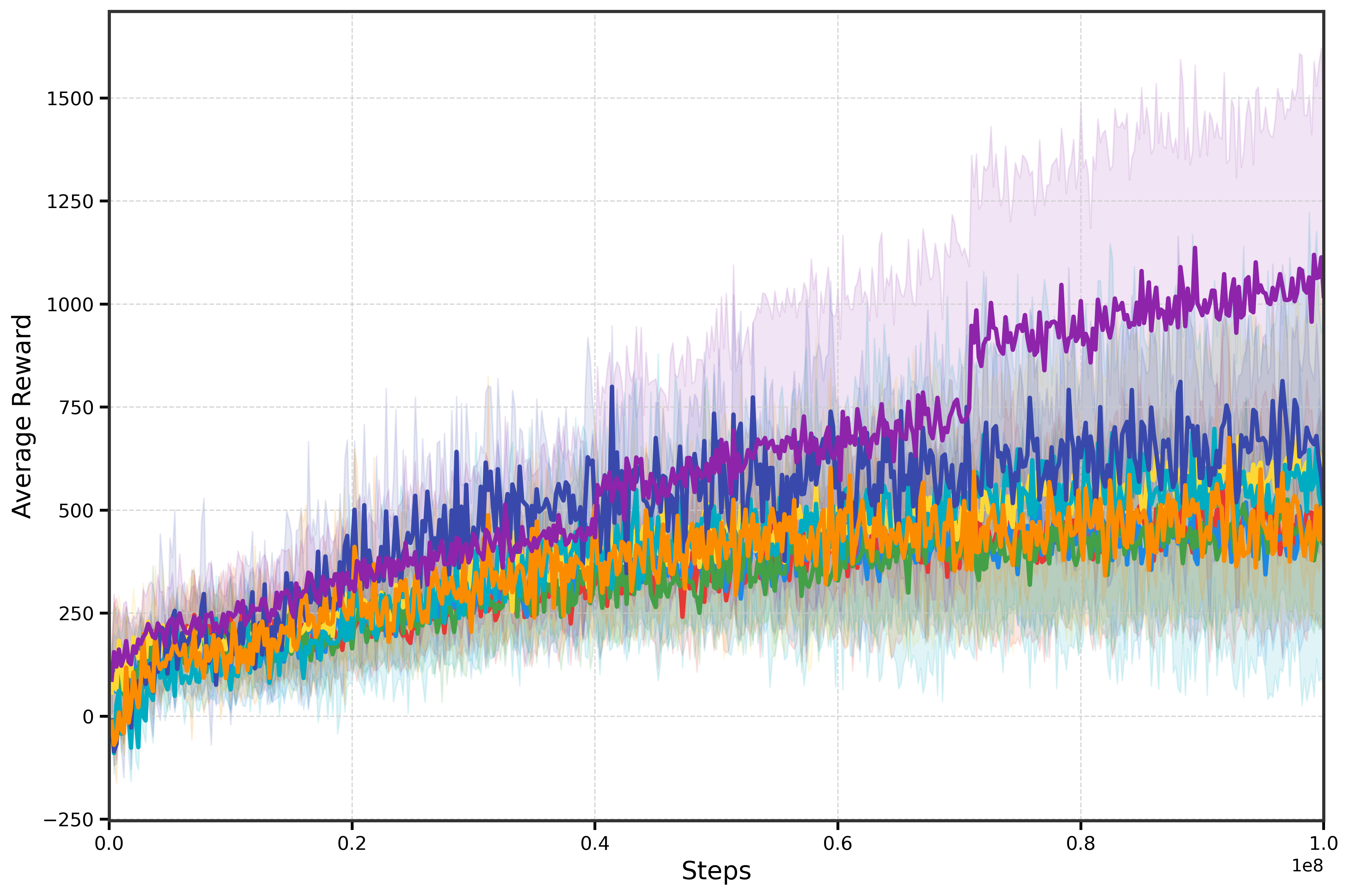}
		\caption{ShadowHand}
		\label{fig:c51_ShadowHand}
	\end{subfigure}
	\caption{\color{red}Learning curve comparisons of FlowCritic and seven PPO-based baselines on 12 control benchmark tasks. Solid lines represent mean episodic returns over 5 random seeds, with shaded areas indicating one standard deviation.}
	\label{fig:all_learning_curves}
		\vspace{-1.5em}
\end{figure*}

\begin{table*}[h]
	\centering
	\caption{\color{red}Final performance comparisons between FlowCritic and baselines.}
	\label{tab:main_performance}
	\scriptsize
	\setlength{\tabcolsep}{2.5pt}
	\renewcommand{\arraystretch}{0.96}
	\resizebox{\textwidth}{!}{%
	\begin{tabular}{@{}lcccccccc@{}}
		\toprule
		\textbf{Env.} & \textbf{PPO} & \textbf{PPO\_AVC} & \textbf{PPO\_CVE} & {\color{c51red}\textbf{PPO\_C51}} & \textbf{PPO\_QD} & \textbf{PPO\_IQN} & \textbf{PPO\_TQC} & \textbf{FlowCritic (Ours)} \\
		\midrule
		FCSk & $11.93 \pm 2.69$ & $13.29 \pm 1.22$ & $14.20 \pm 1.19$ & {\color{c51red}$15.43 \pm 0.35$} & $14.36 \pm 1.09$ & $13.69 \pm 1.24$ & $14.85 \pm 0.91$ & $\mathbf{16.30 \pm 1.63}$ \\
		Quad & $1131.08 \pm 156.59$ & $1271.63 \pm 48.38$ & $1244.01 \pm 39.94$ & {\color{c51red}$1314.43 \pm 28.83$} & $1289.69 \pm 218.54$ & $1310.89 \pm 59.48$ & $1368.31 \pm 187.43$ & $\mathbf{1393.56 \pm 6.95}$ \\
		FCab & $3291.96 \pm 354.19$ & $2822.75 \pm 372.00$ & $2873.50 \pm 326.68$ & {\color{c51red}$1465.24 \pm 126.82$} & $2776.54 \pm 797.23$ & $3484.30 \pm 203.93$ & $2538.07 \pm 874.03$ & $\mathbf{3511.01 \pm 33.21}$ \\
		Ball & $373.88 \pm 2.55$ & $354.01 \pm 33.83$ & $374.25 \pm 2.33$ & {\color{c51red}$373.33 \pm 1.88$} & $374.41 \pm 3.89$ & $376.44 \pm 2.42$ & $370.55 \pm 7.04$ & $\mathbf{377.87 \pm 1.38}$ \\
		SHFF & $269.38 \pm 14.77$ & $283.04 \pm 23.60$ & $282.93 \pm 18.10$ & {\color{c51red}$285.20 \pm 10.97$} & $272.69 \pm 15.78$ & $253.03 \pm 8.03$ & $229.01 \pm 5.05$ & $\mathbf{376.99 \pm 30.59}$ \\
		Anym & $65.03 \pm 4.45$ & $69.13 \pm 0.54$ & $69.13 \pm 0.21$ & {\color{c51red}$66.81 \pm 0.97$} & $69.22 \pm 0.43$ & $66.40 \pm 0.57$ & $58.27 \pm 7.36$ & $\mathbf{69.63 \pm 0.16}$ \\
		Ant & $5340.04 \pm 614.56$ & $7009.38 \pm 835.03$ & $7224.92 \pm 832.91$ & {\color{c51red}$6660.88 \pm 669.20$} & $7269.09 \pm 642.77$ & $8114.85 \pm 357.32$ & $8749.22 \pm 358.81$ & $\mathbf{9345.29 \pm 182.16}$ \\
		AHan & $604.59 \pm 16.04$ & $594.16 \pm 33.88$ & $630.61 \pm 48.91$ & {\color{c51red}$683.99 \pm 17.29$} & $721.03 \pm 142.61$ & $742.51 \pm 112.31$ & $587.01 \pm 4.70$ & $\mathbf{1044.12 \pm 103.53}$ \\
		Huma & $5713.11 \pm 1043.47$ & $6330.85 \pm 394.06$ & $6385.10 \pm 210.36$ & {\color{c51red}$5957.55 \pm 218.79$} & $7196.31 \pm 194.23$ & $5608.35 \pm 2410.01$ & $7880.47 \pm 147.71$ & $\mathbf{8863.61 \pm 114.08}$ \\
		AKuk & $281.86 \pm 114.90$ & $451.39 \pm 87.52$ & $460.17 \pm 84.89$ & {\color{c51red}$453.51 \pm 14.47$} & $344.34 \pm 154.83$ & $419.56 \pm 15.11$ & $432.58 \pm 10.78$ & $\mathbf{529.30 \pm 50.44}$ \\
		AK2A & $248.88 \pm 6.75$ & $264.88 \pm 21.73$ & $259.26 \pm 6.77$ & {\color{c51red}$275.40 \pm 9.10$} & $235.83 \pm 12.80$ & $260.29 \pm 3.70$ & $275.34 \pm 13.25$ & $\mathbf{339.87 \pm 27.00}$ \\
		SHad & $465.08 \pm 237.48$ & $435.71 \pm 175.31$ & $444.01 \pm 165.05$ & {\color{c51red}$466.08 \pm 23.36$} & $594.72 \pm 329.38$ & $554.13 \pm 423.79$ & $658.94 \pm 153.45$ & $\mathbf{1025.76 \pm 427.69}$ \\
		\bottomrule
	\end{tabular}%
	}
\end{table*}

\begin{table*}[h]
	\centering
	\caption{\color{red}Task-averaged CVaR$_{0.1}$ over 12 control tasks. Scores are normalized within each task before averaging. }
	\label{tab:cvar_metric}
	\footnotesize
	\setlength{\tabcolsep}{4pt}
	\renewcommand{\arraystretch}{0.96}
	\begin{tabular*}{\textwidth}{@{\extracolsep{\fill}} lcccccccc @{}}
	\toprule
	\textbf{Metric} & \textbf{PPO} & \textbf{PPO\_AVC} & \textbf{PPO\_CVE} & {\color{c51red}\textbf{PPO\_C51}} & \textbf{PPO\_QD} & \textbf{PPO\_IQN} & \textbf{PPO\_TQC} & \textbf{FlowCritic} \\
	\midrule
	CVaR$_{0.1}$ $\uparrow$& 0.269 & 0.367 & 0.505 & {\color{c51red}0.340} & 0.354 & 0.543 & 0.538 & \textbf{0.968} \\
	\bottomrule
	\end{tabular*}
\end{table*}
\begin{table*}[t]
	\centering
	\caption{\color{red}Training efficiency and resource overhead comparison among different algorithms.}
	\label{tab:efficiency}
	\small
	\setlength{\tabcolsep}{4pt}
	\resizebox{\textwidth}{!}{
	\begin{tabular}{lcccccccc}
	\toprule
	\textbf{Metric} & \textbf{PPO} & \textbf{PPO\_AVC} & \textbf{PPO\_CVE} & {\color{c51red}\textbf{PPO\_C51}} & \textbf{PPO\_QD} & \textbf{PPO\_IQN} & \textbf{PPO\_TQC} & \textbf{FlowCritic} \\
	\midrule
	Train Time (s / 1M steps) & 9.09$\pm$0.14 & 9.23$\pm$0.09 & 11.41$\pm$0.06 & {\color{c51red}9.51$\pm$0.13} & 11.87$\pm$0.05 & 19.56$\pm$0.15 & 13.58$\pm$0.12 & 14.73$\pm$0.13 \\
	Train Memory (GiB) & 2.62 & 2.70 & 3.01 & {\color{c51red}2.70} & 3.01 & 8.92 & 3.47 & 2.72 \\
	Inference Time (s / 1M steps) & 0.207$\pm$0.006 & 0.197$\pm$0.002 & 0.196$\pm$0.001 & {\color{c51red}0.199$\pm$0.001} & 0.195$\pm$0.001 & 0.196$\pm$0.001 & 0.197$\pm$0.001 & 0.200$\pm$0.001 \\
	Inference Memory (GiB) & 2.42 & 2.42 & 2.43 & {\color{c51red}2.4} & 2.43 & 2.42 & 2.43 & 2.42 \\
	\bottomrule
	\end{tabular}}
\end{table*}

\noindent
\subsection{Value Estimation in a Single-Step Environment}
In this experiment, following the setup in TQC~\cite{TQC}, we construct a single-step decision process to evaluate the performance of point estimation and FlowCritic's flow matching approach in value estimation. The environment is formalized as $\mathcal{M} = (\mathcal{S}, \mathcal{A}, \rho, R, \gamma)$, where the state space is defined as $\mathcal{S} = [-3, 3]^2$, and states $s=(x,y)$ are sampled from a uniform distribution $\rho$. The action space is given by $\mathcal{A} = \{a_0\}$, and each episode terminates after a single step, with the reward depending solely on the current state.
The true expected value function of $\mathcal{M}$ is defined as:
\begin{equation}
	V^*(x,y)
	= 1.5\,\exp\!\left(-\frac{\|s-\mu_1\|^2}{2\ell_1^{\,2}}\right)
	- 1.5\,\exp\!\left(-\frac{\|s-\mu_2\|^2}{2\ell_2^{\,2}}\right),
\end{equation}
where $\mu_1=(0.8,0.8)$, $\mu_2=(-0.8,-0.8)$, and $\ell_1 = \ell_2 = 0.5$. To introduce stochasticity into the environment, the observed reward is defined as:
{\small
	\begin{equation}
		\tilde{R}(s)\sim
		\begin{cases}
			(1-\varepsilon)\,\mathcal{N}\!\big(V^*(s),\,\sigma_1^2\big)\;+\;\varepsilon\,\mathcal{N}\!\big(V^*(s),\,\sigma_2^2\big), & \|s\|\le 1,\\[3pt]
			\mathcal{N}\!\big(V^*(s),\,\sigma_3^2\big), & \text{otherwise},
		\end{cases}
		\label{eq:reward}
	\end{equation}
}where $\|s\|=\sqrt{x^2+y^2}$. The region $\|s\|\le 1$ represents high stochasticity with $\varepsilon=0.2$, $\sigma_1=0.05$, $\sigma_2=10$, and $\sigma_3=0.01$.
We compare the value estimation capabilities of FlowCritic and point estimation methods within the constructed environment. To ensure a fair comparison, both methods adopt an identical neural network architecture, and are trained using data sampled from $\mathcal{S} = [-2, 2]^2$. Figures~\ref{fig:toy2} and ~\ref{fig:toy3} present the absolute errors for point estimation and FlowCritic, respectively. In the high-stochasticity region $\|s\| \leq 1$, point estimation suffers from large errors that exceed 0.15 due to outlier noise, while FlowCritic maintains robust predictions with substantially lower errors throughout the state space, attributed to the advantages of flow matching's multiple sampling and value truncation mechanisms.

To further analyze the distributional properties captured by FlowCritic, we compute and visualize its CoV across the state space as shown in Fig.~\ref{fig:toycv} using Eq.~\eqref{eq:cv}. The CoV heatmap reveals two important characteristics. First, high CoV regions closely match the high noise region $\|s\| \le 1$ within the training region $\mathcal{S} = [-2, 2]^2$, confirming that CoV effectively measures noise levels across samples. This alignment justifies the weighting mechanism in Eq.~\eqref{eq:weight_cv} for suppressing high-noise samples. 
Second, we observe that the CoV also increases significantly in the out-of-training regions (outside $\mathcal{S} = [-2, 2]^2$), where insufficient training data leads to increased noise in the generated value estimates. This further validates the rationality of the weighting strategy in Eq.~\eqref{eq:weight_cv}, which can adaptively suppress high-noise samples.
 \begin{figure}[h]
 	\centering
 	\includegraphics[width=0.95\linewidth]{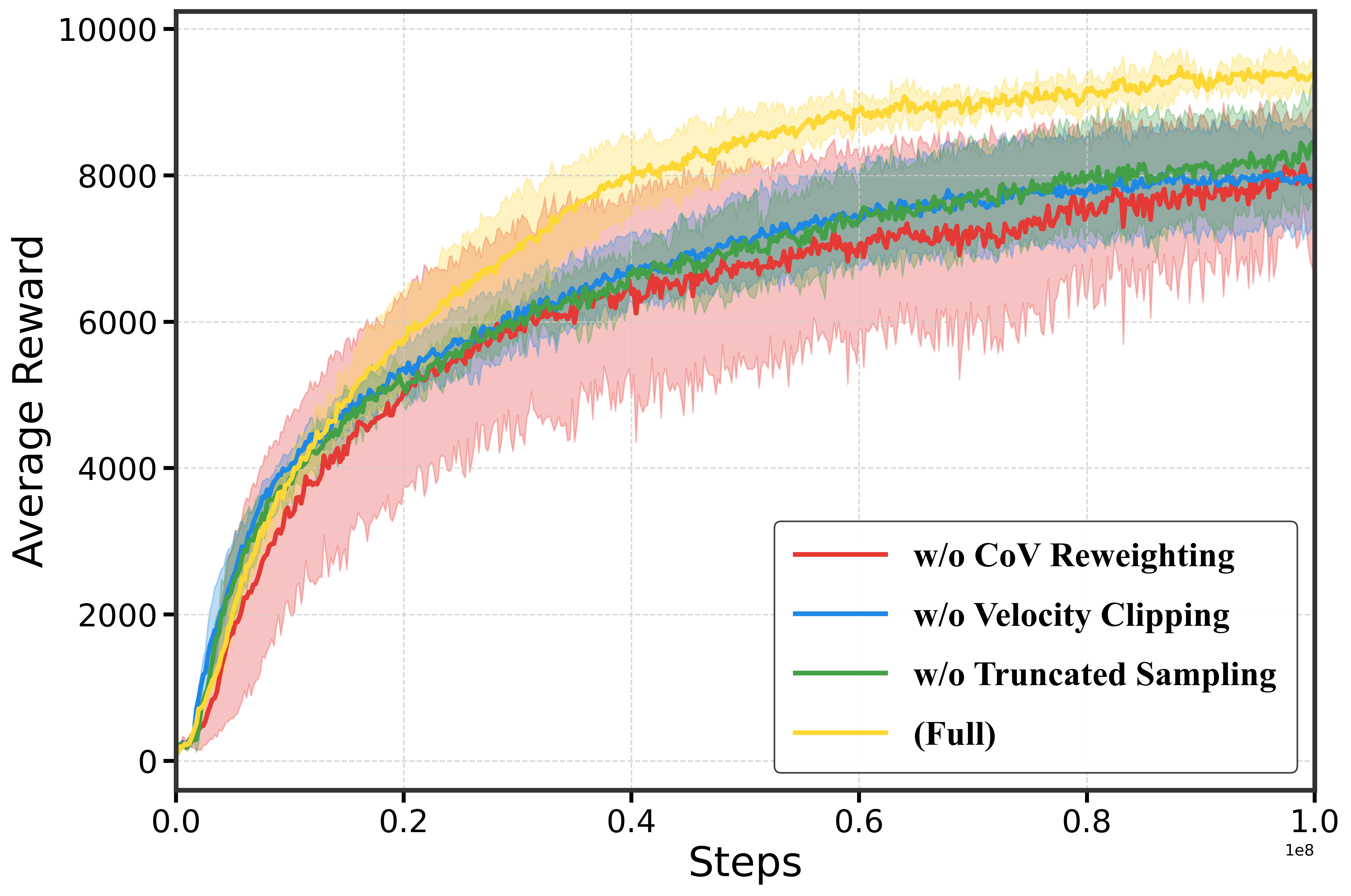}
 	\caption{Ablation study on key components of FlowCritic.}
 	\label{fig:ablation_components}
 	\vspace{-1em}
 \end{figure}
 \begin{figure*}[h] 
    \centering 
    \begin{subfigure}[b]{0.4\textwidth} 
        \centering 
        \includegraphics[width=\textwidth]{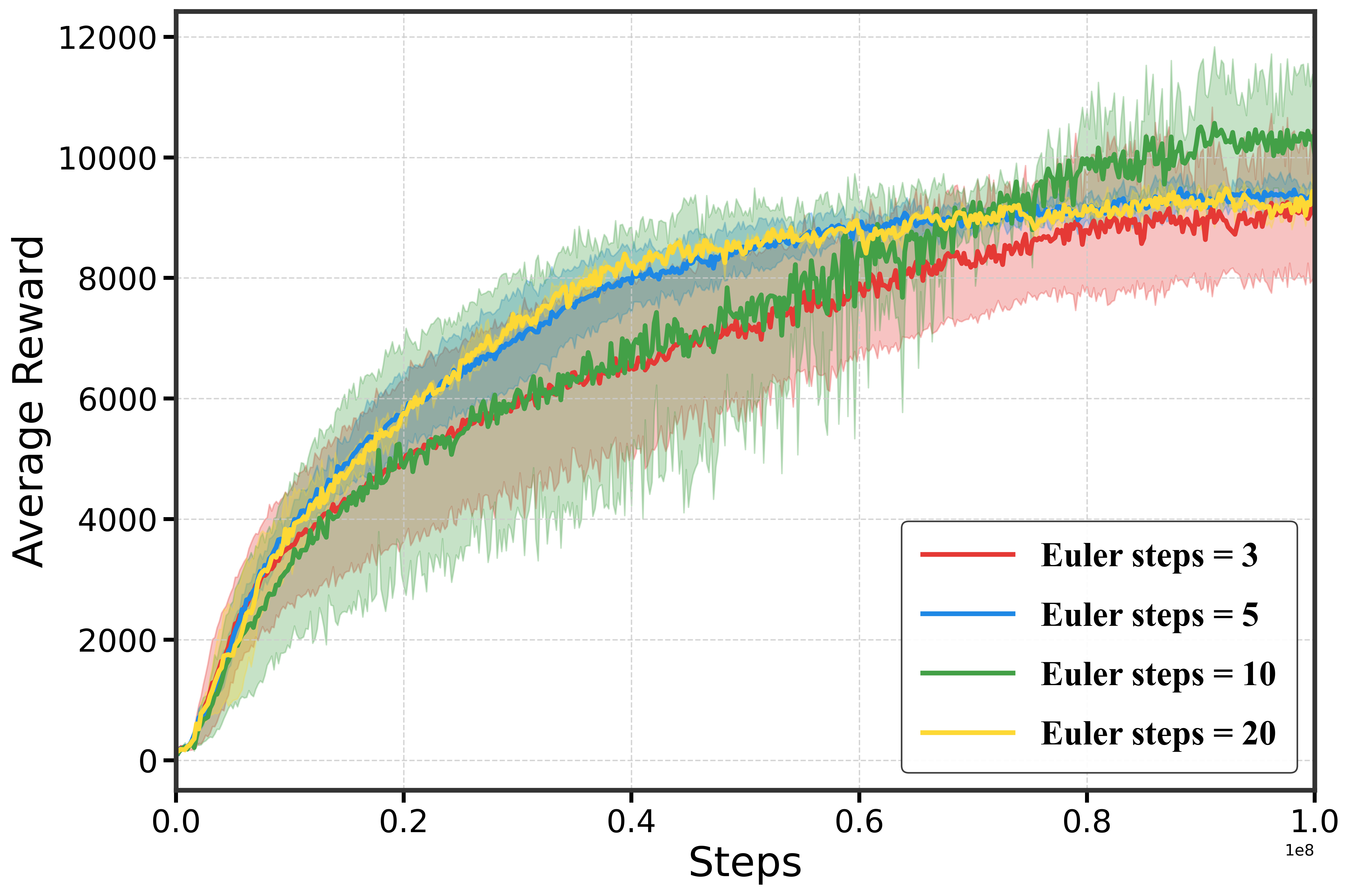} 
        \caption{Sensitivity to Euler steps} 
        \label{fig:ablation_euler} 
    \end{subfigure} 
    \begin{subfigure}[b]{0.4\textwidth} 
        \centering 
        \includegraphics[width=\textwidth]{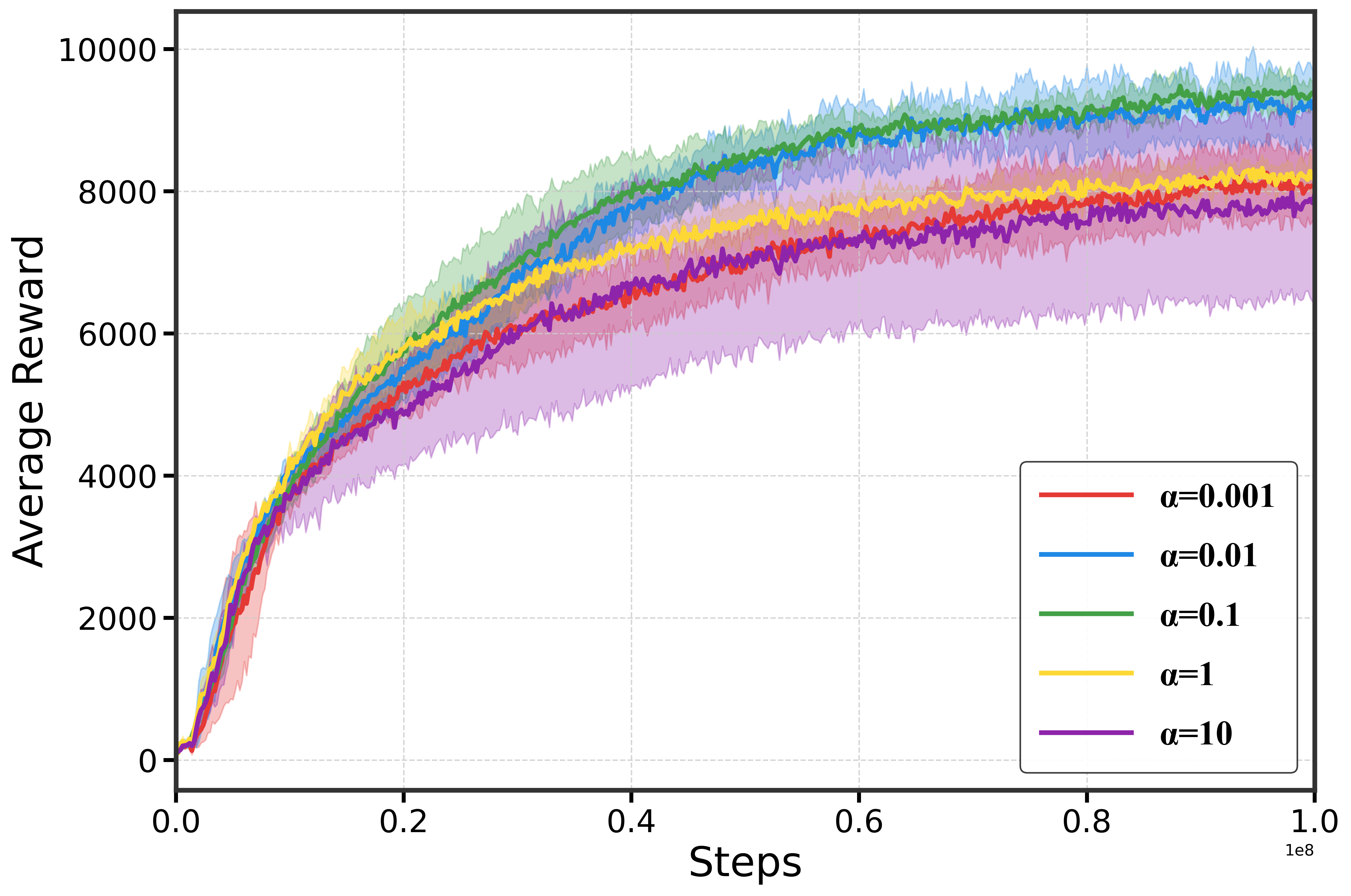} 
        \caption{Sensitivity to CoV temperature $\alpha$} 
        \label{fig:ablation_alpha} 
    \end{subfigure} 
    \\[-0.2em]
    \begin{subfigure}[b]{0.4\textwidth} 
        \centering 
        \includegraphics[width=\textwidth]{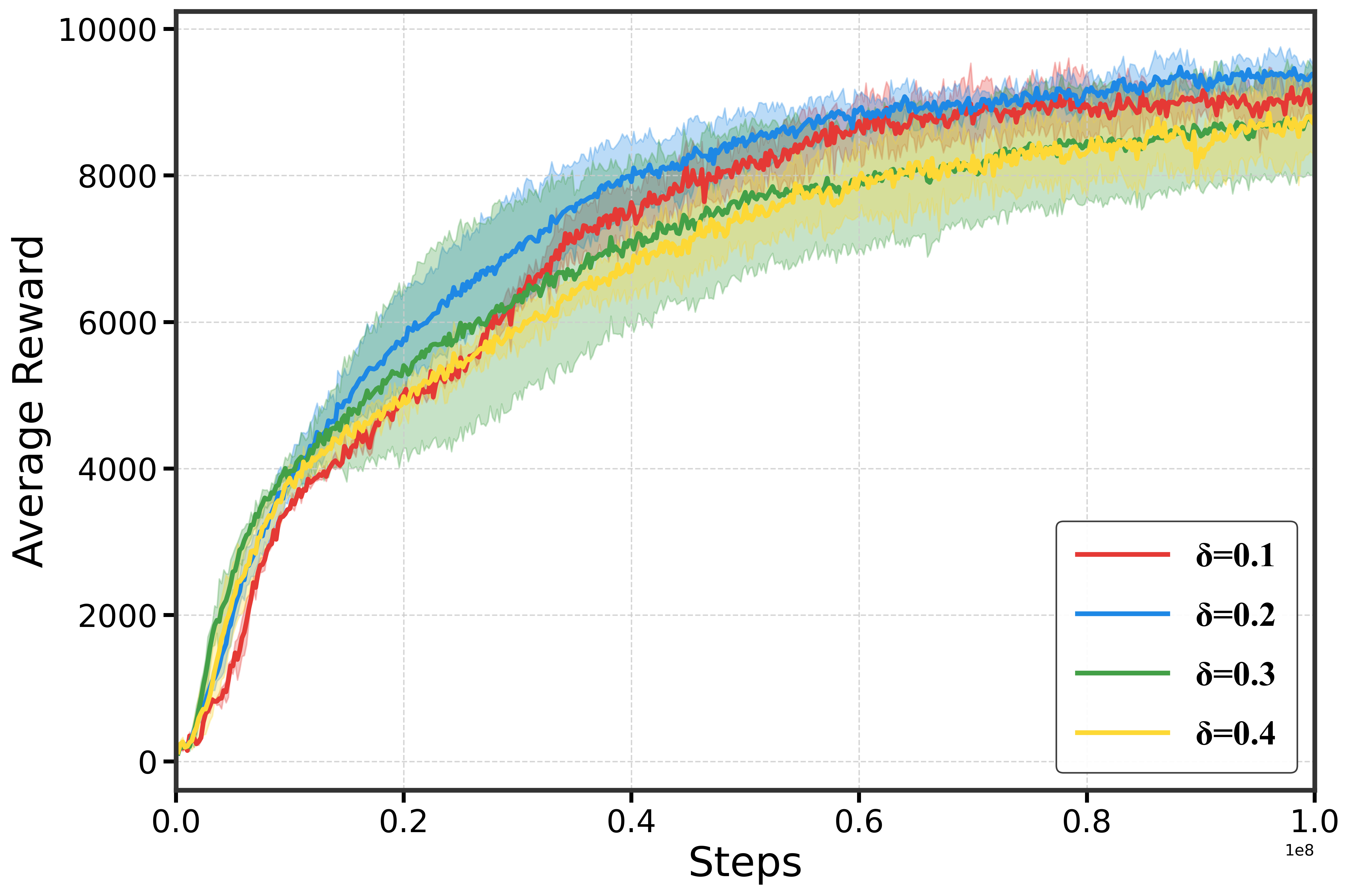} 
        \caption{Sensitivity to clipping threshold $\delta$} 
        \label{fig:ablation_delta} 
    \end{subfigure} 
    \begin{subfigure}[b]{0.4\textwidth} 
        \centering 
        \includegraphics[width=\textwidth]{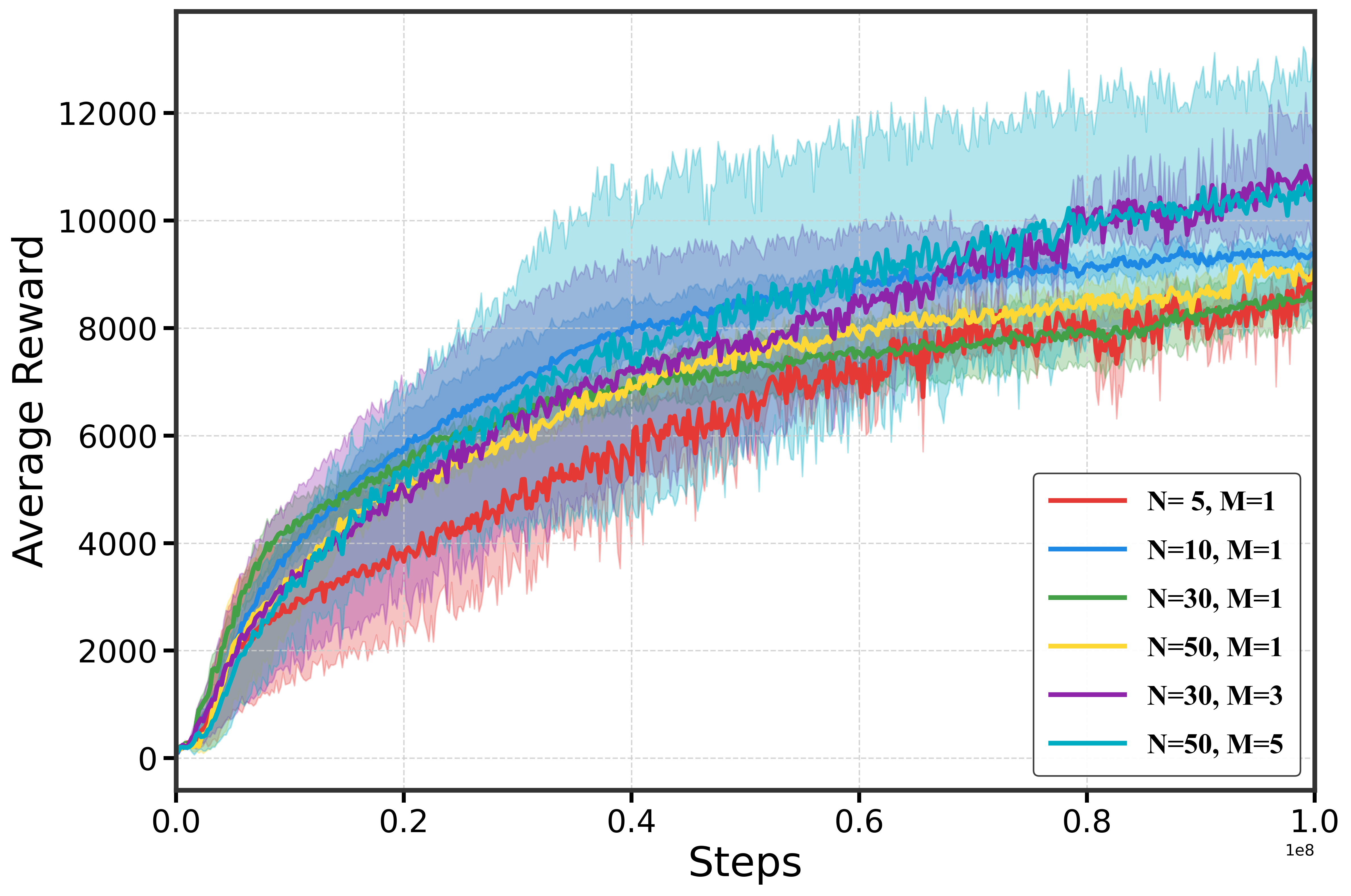} 
        \caption{Sensitivity to $N$ and truncation $M$} 
        \label{fig:synergy_n_m} 
    \end{subfigure} 
    \caption{Sensitivity analysis of key hyperparameters of FlowCritic.} 
    \label{fig:ablation_hyperparams} 
    \vspace{-1em} 
\end{figure*}
\subsection{Experimental Setup for Continuous Control Tasks}
We evaluate FlowCritic across 12 control tasks spanning a wide range of complexity and dimensionality. As detailed in Table \ref{tab:env_details}, our benchmark environments encompass diverse challenges including complex locomotion, dexterous manipulation, and robotic arm control~\cite{todorov2012mujoco}, with observation dimensions ranging from 19 to 211 and action dimensions from 3 to 46. All experiments utilize NVIDIA Isaac Gym, which provides GPU-accelerated parallel simulation~\cite{makoviychuk2021isaac}. Since FlowCritic extends the on-policy policy gradient paradigm and Isaac Gym's synchronous, high-throughput data collection is inherently optimized for on-policy algorithms, we focus our evaluation on on-policy baselines rather than off-policy methods such as SAC or TD3~\cite{SAC, td3}. 

\begin{revblock}
To provide a comprehensive evaluation, we further adapt several representative value estimation methods to the PPO framework. To the best of our knowledge, systematic studies of these methods within PPO remain limited in the existing literature. Specifically, we implement the following variants: PPO with Averaged Value Critics (PPO\_AVC), which averages multiple independent critics to reduce estimation variance~\cite{robotkeyframing, averagedqn}; PPO with Conservative Value Ensemble (PPO\_CVE), which adopts the conservative strategy of SAC-N to improve the robustness of value estimation~\cite{sacn}; PPO with Quantile Distribution (PPO\_QD), which incorporates the quantile modeling idea of QR-DQN to learn the full value distribution through quantile regression~\cite{QR-DQN}; PPO with Implicit Quantile Networks (PPO\_IQN), which uses implicit quantile parameterization to model value distributions over continuous quantile fractions~\cite{IQN}; {\color{c51red}PPO with Categorical Value Distribution (PPO\_C51), which discretizes the return range into a fixed set of atoms and learns a categorical value distribution following C51~\cite{bellemare};} and PPO with Truncated Quantile Critics (PPO\_TQC), which employs multiple critics for quantile based value distribution modeling and improves the robustness of critic learning~\cite{TQC}.

For a fair comparison, the method-specific hyperparameters of all algorithms are set according to the recommendations in the corresponding original papers. Full details are provided in Table~\ref{tab:hyperparams}.
\end{revblock}
Algorithm performance is measured by average episodic return. For statistical robustness, all results are aggregated over five independent runs with different random seeds. In tabular results, we report performance over the final 10\% of training. In learning curves, we present mean return evaluated every 200,000 steps, with shaded regions representing one standard deviation. {\color{c51red}Here, one step denotes one environment transition. A synchronized simulator step advances all $N_e$ parallel environments and contributes $N_e$ steps.}

\subsection{Performance Comparisons}
\begin{revblock}
Figure~\ref{fig:all_learning_curves} and Table~\ref{tab:main_performance} report the learning curves and final performance on 12 control benchmark environments. {\color{c51red}Among the baseline methods, PPO\_AVC, PPO\_CVE, PPO\_QD, PPO\_IQN, PPO\_TQC, and PPO\_C51 generally outperform standard PPO on many tasks. PPO\_TQC is the strongest baseline overall and achieves the best baseline result on 4 of the 12 environments, while PPO\_IQN and PPO\_C51 are each the best baseline on 3 environments.} However, the baseline methods still show clear inconsistency across tasks, and some variants even perform worse than the standard PPO on certain environments. This result suggests that directly adapting existing value estimation methods to PPO does not always yield robust gains.
FlowCritic outperforms all baseline methods on all 12 environments, often by a clear margin. The improvement is particularly pronounced on challenging high dimensional control tasks such as Ant, AllegroHand, and ShadowHand. FlowCritic also tends to achieve smaller standard deviations on relatively lower dimensional environments, which indicates better training stability. These results show that FlowCritic provides a more effective and robust value estimation mechanism by modeling the return distribution with flow matching and leveraging the learned distributional statistics for policy optimization.
In addition, Table~\ref{tab:cvar_metric} reports the task-averaged $\mathrm{CVaR}_{0.1}$ over 12 control tasks. FlowCritic achieves the highest average score, indicating stronger performance in the low-return tail.
\end{revblock}

\begin{revblock}
Furthermore, we compare training efficiency and resource overhead in Table~\ref{tab:efficiency}. For a fair comparison, all methods are evaluated under the same training budget with 1024 parallel environments. FlowCritic performs 5-step Euler ODE integration for value estimation, introducing additional training overhead compared with PPO. Its training cost is close to PPO\_TQC, which similarly involves multiple quantile estimates and truncated-target construction, while PPO\_IQN takes longer due to quantile embedding and quantile regression. Unlike these distributional baselines, which maintain multiple quantile estimates or additional quantile embeddings in parallel, FlowCritic performs the Euler iterations sequentially and stores only a small number of scalar return samples, so its additional memory overhead is small. Notably, the policy inference latency of FlowCritic is comparable to that of PPO, because deployment executes only the actor for online control, while the critic and its ODE sampling are not invoked. Therefore, the additional training cost of FlowCritic does not increase the control latency in real-world deployment.
\end{revblock}
\subsection{Ablation Studies}

To gain deeper insights into the contribution of each FlowCritic component to the final performance, we conduct systematic ablation studies on the Ant environment. We select Ant because it exhibits moderate complexity and favorable training stability, enabling clear observation of individual component effects while other environments demonstrate similar trends.

\subsubsection{Core Component Ablation}
\begin{figure}[t]
    \centering
    \includegraphics[width=0.48\textwidth]{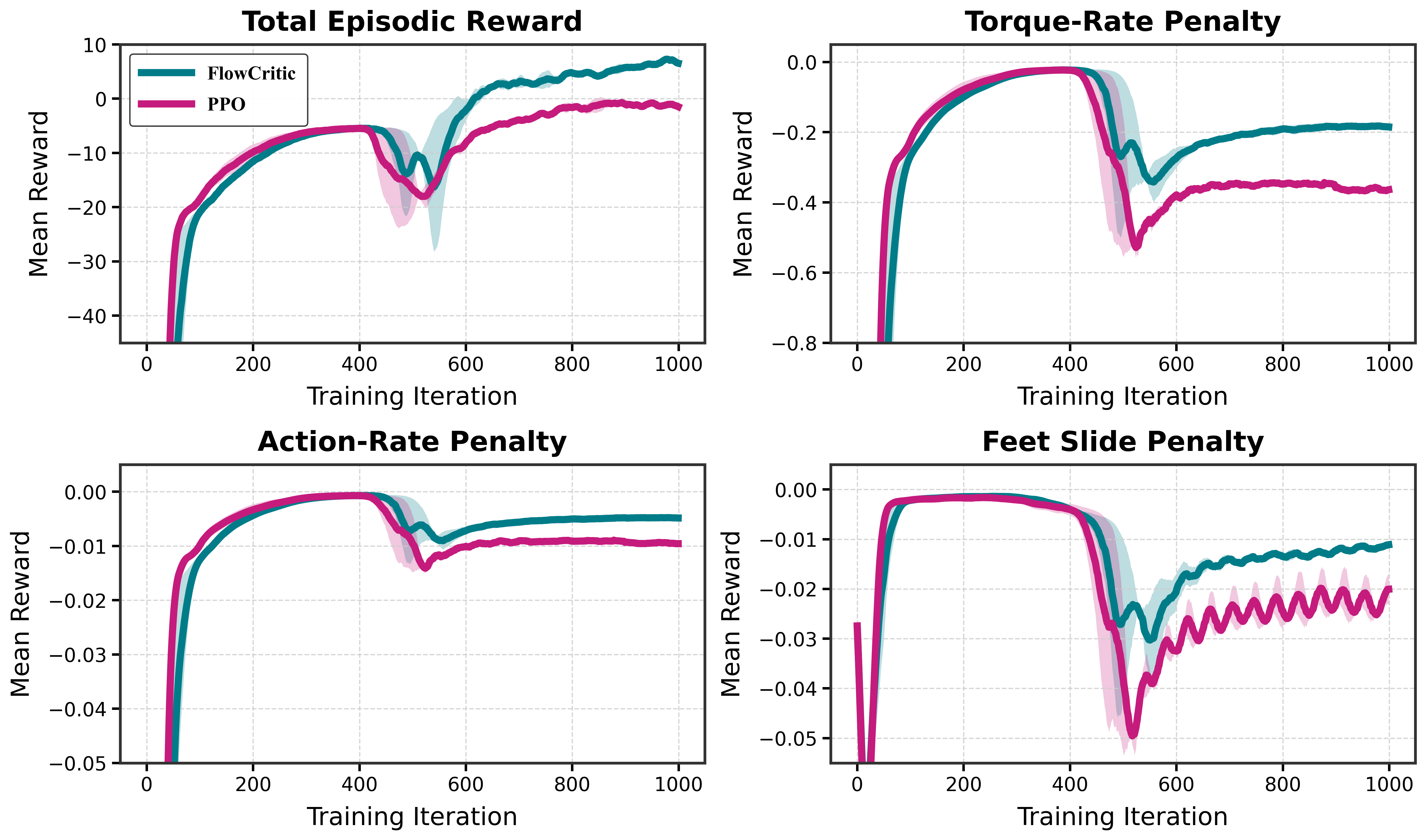}
    \caption{\color{red}Comparison on the quadrupedal locomotion task.}
    \label{fig:hardware_statistical_comparison}
    \vspace{-1em}
\end{figure}
\begin{figure*}[h]
	\centering
	\begin{subfigure}[b]{\linewidth}
		\centering
		\includegraphics[width=\linewidth]{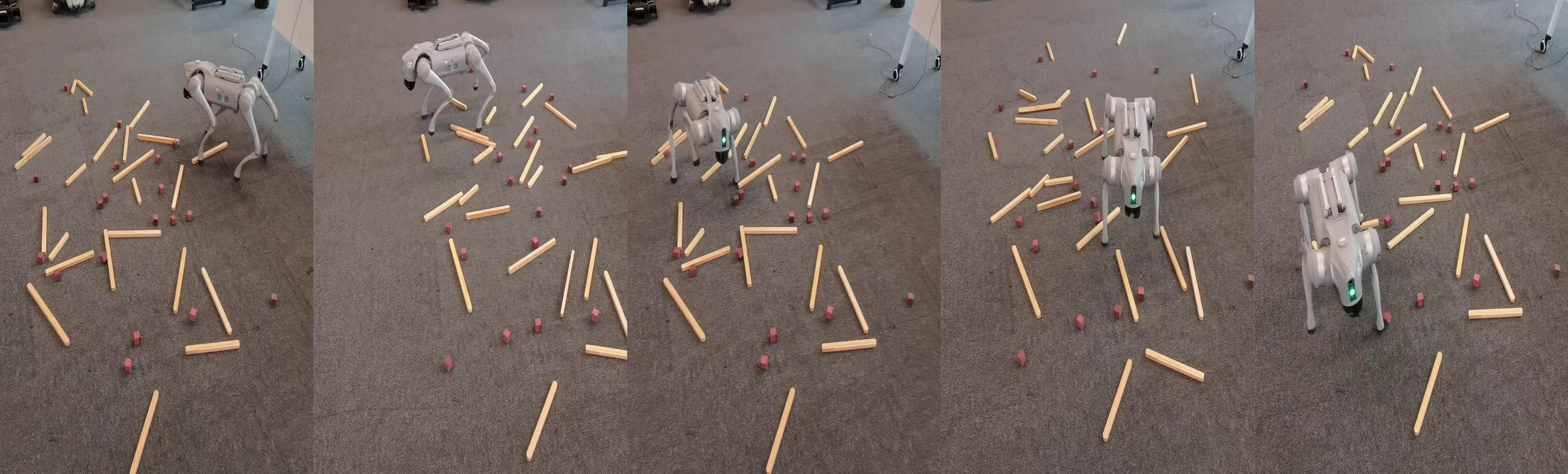}
		\caption{Omnidirectional locomotion in cluttered environment}
		\label{fig:robot_obstacle}
	\end{subfigure}
	\vspace{2mm}
	\begin{subfigure}[b]{\linewidth}
		\centering
		\includegraphics[width=\linewidth]{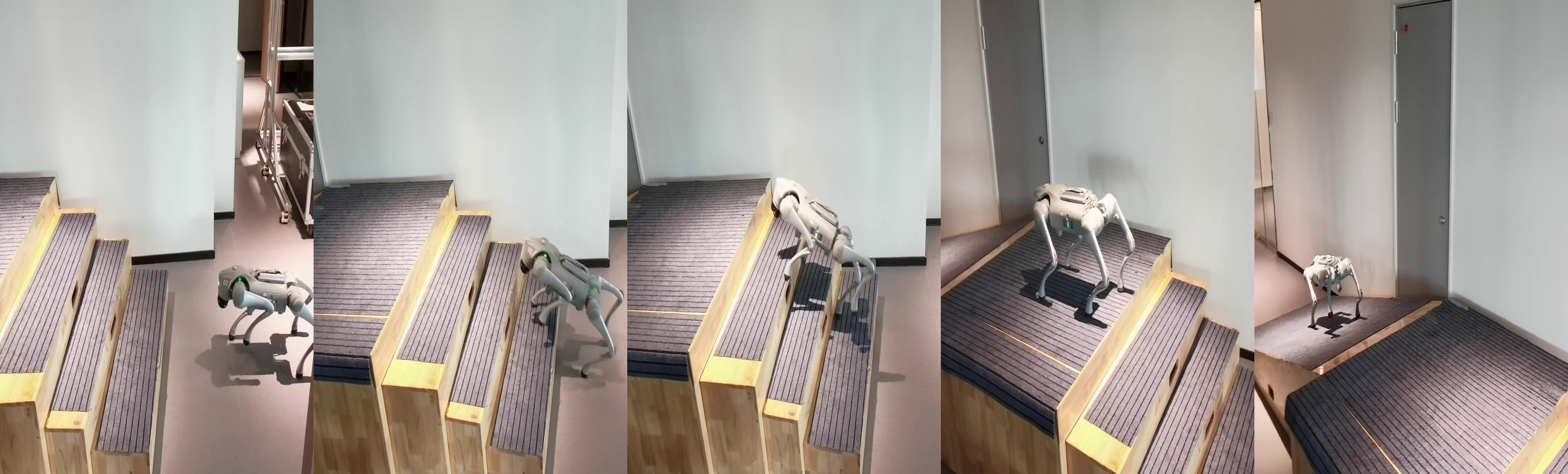}
		\caption{Stair climbing on stepped platforms}
		\label{fig:robot_stairs}
	\end{subfigure}
	\caption{Real-world deployment on Unitree Go2 quadrupedal robot.}
	\label{fig:real_robot}
		\vspace{-1em}
\end{figure*}

We sequentially remove key components of FlowCritic to assess their impact on performance, with results presented in Fig. \ref{fig:ablation_components}. Removing CoV reweighting causes the most substantial performance drop throughout training, validating its critical role in improving sample efficiency. In contrast, removing velocity clipping or truncated sampling maintains comparable early-stage performance but exhibits progressively increased variance (wider shaded regions), highlighting their importance for long-term training stability.

\subsubsection{Hyperparameter Sensitivity and Component Synergy}

Building upon the validation of component effectiveness, we conduct sensitivity analysis of key hyperparameters, with results presented in Fig.~\ref{fig:ablation_hyperparams}.
\begin{revblock}
Figure~\ref{fig:ablation_euler} shows the sensitivity to Euler integration steps. Since FlowCritic samples the value distribution by integrating the flow-matching velocity field, too few Euler steps may cause coarse ODE discretization and reduce sampling fidelity. Increasing the steps to 10 achieves the best result in this experiment, while 20 steps brings no further gain. A finer solver follows the learned velocity field more closely, but may also accumulate its approximation errors along the ODE trajectory, offsetting the benefit of reduced discretization error. Thus, we use Euler steps $=5$ as the default setting to balance performance and efficiency.
\end{revblock}
Figure~\ref{fig:ablation_alpha} illustrates the impact of CoV temperature parameter $\alpha$, which controls the concentration of weight. Larger values strongly suppress the weights of high-noise samples, while smaller values yield nearly uniform weighting. The results show that $\alpha=0.1$ and $\alpha=0.01$ exhibit similar learning curves throughout training and achieve optimal performance.
Figure~\ref{fig:ablation_delta} demonstrates the sensitivity to velocity field clipping threshold $\delta$. Smaller thresholds impose stronger constraints on gradient updates, while larger thresholds relax these restrictions. The experiments show that $\delta=0.2$ and $\delta=0.1$ achieve comparable final performance.
Based on these observations, we set $\alpha=0.1$ and $\delta=0.2$ as default hyperparameters throughout this work.

Furthermore, Fig.~\ref{fig:synergy_n_m} reveals significant synergistic effects between the number of sampled value estimates $N$ and truncation size $M$. We evaluate combinations of $N \in \{5, 10, 30, 50\}$ and $M \in \{1, 3, 5\}$. The results show that naively increasing $N$ alone does not monotonically improve performance, as $N=50, M=1$ even underperforms $N=10, M=1$, since larger sample size $N$ may introduce additional outlier estimates that degrade learning quality without proper truncation. In contrast, jointly scaling both $N$ and $M$ yields substantial performance gains, with $N=30, M=3$ achieving notably superior returns while further increasing to $N=50, M=5$ provides only limited marginal improvement. This demonstrates that effective synergy requires pairing larger sample size with proportional truncation to filter noise.
Balancing computational cost against performance, we adopt $N=10, M=1$ as the default configuration throughout our experiments.
\subsection{Real Quadrupedal Robot Deployment}
We deploy trained FlowCritic policies on the Unitree Go2 quadrupedal robot to validate their practical effectiveness. Policies trained in Isaac Gym transfer directly to hardware without sim-to-real adaptation, executing at 50Hz with a 454-dimensional state space that incorporates observations and privileged information, and a 12-dimensional action space for joint control.
\begin{revblock}
	As shown in Fig.~\ref{fig:hardware_statistical_comparison}, we compare FlowCritic with the PPO baseline under the same setting. FlowCritic attains a higher total reward, and its torque-rate, action-rate, and feet-slide penalties are all closer to zero. These results indicate that FlowCritic produces smoother control commands and reduces foot sliding, leading to higher-quality locomotion behavior.
	\end{revblock}
	Building on this, Fig.~\ref{fig:real_robot} further illustrates the real-world performance of FlowCritic in two representative scenarios. Figure~\ref{fig:robot_obstacle} shows omnidirectional locomotion in a cluttered environment with wooden sticks and spherical obstacles, where the robot maintains a stable gait and agile movement. Figure~\ref{fig:robot_stairs} shows stair climbing on a stepped platform, where the robot adapts to elevation changes while keeping balance.
\section{Conclusion}
In this paper, we have introduced the FlowCritic, a novel framework that is the first to apply generative models for value distribution representation in RL.  FlowCritic introduces innovations in two key areas: the flexible representation of the value distribution and the effective utilization of its statistical information. It employs a flow matching generative model to accurately capture arbitrarily complex value distributions, overcoming the representation limitations of conventional methods. Concurrently, FlowCritic pioneers an adaptive weighting mechanism based on the CoV that integrates statistical properties of the value distribution beyond its expectation into the policy update process, thereby achieving more robust policy updates. 
The superiority and practicality of FlowCritic are substantiated by both theoretical analysis and extensive experiments on challenging control benchmarks and a physical quadruped robot, demonstrating performance that significantly surpasses existing baseline algorithms.
\begin{revblock}
Future work will explore more efficient generative models to reduce the training overhead of FlowCritic. Extending FlowCritic to high-dimensional observation spaces, such as image-based locomotion and manipulation, is also a promising direction. In addition, we will examine its generality under off-policy algorithms such as SAC and DDPG, and investigate optimal transport based path construction to improve the expressiveness of value distribution modeling.
\end{revblock}
In addition, applying FlowCritic to continual learning scenarios, where value distributions evolve with changing environments, represents a promising direction for lifelong learning systems.
\label{con}

\bibliographystyle{IEEEtran}  
\bibliography{ref}

\end{document}